\PassOptionsToPackage{authoryear}{natbib}
\documentclass[11pt]{article}

\usepackage[dvipsnames]{xcolor}

\usepackage[
    colorlinks=true,
    linkcolor=MidnightBlue,
    citecolor=MidnightBlue,
    urlcolor=RoyalBlue
]{hyperref}
\usepackage[a4paper,top=2.45cm,bottom=2.55cm,left=2.55cm,right=2.55cm]{geometry}
\usepackage{lmodern}
\usepackage{wrapfig}
\usepackage[utf8]{inputenc} % allow utf-8 input
\usepackage[T1]{fontenc}    % use 8-bit T1 fonts
\usepackage{url}            % simple URL typesetting
\usepackage{booktabs}       % professional-quality tables
\usepackage{amsfonts}       % blackboard math symbols
\usepackage{nicefrac}       % compact symbols for 1/2, etc.
\usepackage{microtype}      % microtypography

\usepackage{float}
\usepackage{array}
\usepackage{caption}

\usepackage{graphicx}
\usepackage{subcaption}
\usepackage{bm}

\usepackage{amsmath}
\usepackage{amssymb}
\usepackage{mathtools}
\usepackage{amsthm}
\usepackage[shortlabels]{enumitem}

\usepackage[capitalize,noabbrev]{cleveref}

\theoremstyle{plain}
\newtheorem{theorem}{Theorem}[section]
\newtheorem{proposition}[theorem]{Proposition}
\newtheorem{lemma}[theorem]{Lemma}
\newtheorem{fact}[theorem]{Fact}
\newtheorem{corollary}[theorem]{Corollary}
\theoremstyle{definition}
\newtheorem{definition}[theorem]{Definition}
\newtheorem{assumption}[theorem]{Assumption}
\newtheorem{algorithm}[theorem]{Algorithm}
\theoremstyle{remark}
\newtheorem{remark}[theorem]{Remark}

\usepackage[textsize=tiny]{todonotes}
\newcommand{\norm}[1]{\lVert{#1}\rVert}
\newcommand{\abs}[1]{\lvert{#1}\rvert}
\newcommand{\scalarp}[1]{\langle{#1}\rangle}
\newcommand{\pair}[2]{{\langle{{#1},{#2}}\rangle}}
\DeclareMathOperator*{\argmin}{argmin}

\DeclareMathOperator{\prox}{prox}
\newcommand{\HH}{X}
\newcommand{\bHH}{\bm{X}}
\newcommand{\R}{\mathbb{R}}

\newcommand{\x}{x}
\newcommand{\m}{d}
\newcommand{\C}{\mathcal L}
\newcommand{\bx}{\bm{x}}
\newcommand{\by}{\bm{y}}
\newcommand{\y}{y}
\newcommand{\buu}{\bm{g}}
\newcommand{\uu}{g}
\newcommand{\w}{w}
\newcommand{\iter}{n}
\newcommand{\ite}{k}
\newcommand{\N}{\mathbb{N}}
\newcommand{\RX}{\left]-\infty, +\infty\right]}
\newcommand{\reg}{\varphi}
\newcommand{\Id}{\mathrm{Id}}
\newcommand{\W}{W}
\newcommand{\EE}{\mathbb E}

\title{AdaGrad-Diff: a Corrected Version of the Adaptive Gradient Algorithm for Composite Optimization}

\date{\vphantom{October 5, 2026}}

\author{%
  Matia Bojovi\'c$^{1,2}$
  \qquad
  Saverio Salzo$^{1,3}$
  \qquad
  Massimiliano Pontil$^{1,4}$
  \\[3ex]
  $^{1}$Computational Statistics and Machine Learning,
  Istituto Italiano di Tecnologia
  \\
  $^{2}$Department of Mathematics, University of Genoa, Genoa, Italy
  \\
  $^{3}$DIAG, Sapienza University of Rome, Rome, Italy
  \\
  $^{4}$Department of Computer Science,
  University College London, London, UK
}

\usepackage{natbib}

\begin{document}

\maketitle

\begin{abstract}
\noindent
Vanilla gradient methods are often highly sensitive to the choice of stepsize, which typically requires manual tuning. Adaptive methods alleviate this issue, with AdaGrad being a prominent example. In this work, we first show that, in composite convex optimization, AdaGrad may fail to exploit the H\"older-smoothness of the differentiable term, as its stepsizes may decrease too rapidly, leading to convergence rates arbitrarily close to the nonsmooth $\mathcal O(n^{-1/2})$ rate.
Motivated by this limitation, we study AdaGrad-Diff, a block-wise extension of the UniSgd method considered by \cite{rodomanov2024universality}, which accumulates successive gradient differences rather than gradients themselves. In the deterministic Lipschitz-smooth setting, we prove summability of squared gradient differences without assuming bounded iterates, yielding weak convergence to a minimizer in Hilbert spaces. Moreover, we show that, under $\nu$-H\"older-smoothness, stepsizes decay no faster than $n^{-(1-\nu)/2}$, thereby avoiding AdaGrad's failure mechanism and recovering the standard $\mathcal O(n^{-(1+\nu)/2})$ rate. The same approach recovers known stochastic convergence rates under a unified condition covering bounded variance and interpolation. Numerical experiments demonstrate that the proposed method is more robust than AdaGrad in several practically relevant settings.
\end{abstract}

\allowdisplaybreaks
\section{Introduction} 

Gradient-based optimization methods are among the most widely used tools in machine learning due to their simplicity and broad applicability. In particular, Gradient Descent (GD) and its stochastic variants remain the canonical first-order methods for large-scale optimization problems. GD generates iterates according to
$\bx^{\iter+1} = \bx^\iter - \eta_\iter \buu^\iter$, starting from an arbitrary initial point $\bx^0$, where $\buu^\iter$ denotes a (sub)gradient of the objective function evaluated at $\bx^\iter$ and $\eta_\iter$ is the stepsize. The stepsize $\eta_\iter$ is the key parameter governing convergence speed and stability. 
Classic convergence analysis of the GD algorithm relies on specific schedules for the stepsize $ \eta_\iter $, depending on the underlying level of smoothness of the function. If the objective function $f$ is either Lipschitz-continuous or H\"older-smooth with exponent $\nu$ and constant $L$, the stepsize choice relies on conditions of the form
\[ \eta_n = \frac{\eta}{\sqrt n} \qquad\text{or} \qquad
{\eta_\iter}
= \frac{\eta}{n^{(1-\nu)/2}},
\]
respectively, where $\eta$ requires knowledge of the smoothness constant when $f$ is $L$-smooth.
However, the quantities $\nu$, $L$ appearing above are typically unknown in practice, making principled
stepsize selection difficult and often requiring extensive hyperparameter tuning. Moreover, the choice of $\eta$ is notoriously delicate: too small a value leads to slow progress, while too large a value may prevent convergence altogether. 

To overcome these limitations, a broad body of work has focused on adaptive gradient methods, which automatically adjust stepsizes based on the observed optimization history. Among these methods, AdaGrad, by \cite{duchi2011adaptive}, has played a particularly influential role, giving rise to a large family of adaptive algorithms. Its central idea is to replace a fixed or pre-scheduled stepsize by a sequence obtained by rescaling a base parameter $\eta>0$ using accumulated information from past gradients. One of the defining features of AdaGrad is that the stepsize adaptation is performed at the level of individual components. In its most general form, the iterates satisfy
\begin{equation}\label{adagrad_uncostrained_update}
x_i^{\iter+1} = x_i^\iter - \frac{\eta}{\w_i^\iter}\,\uu_i^\iter,
\end{equation}
where each $\w_i^\iter$ is a nondecreasing quantity that accumulates the squares of the $i$-th components of the gradients along the optimization trajectory. As $\w_i^\iter$ grows, the corresponding scaling factor $\eta/\w_i^\iter$ decreases, automatically moderating progress in directions where gradients accumulate rapidly. This component-wise adaptation enables AdaGrad to adjust to the geometry of the problem, alleviating the need to pre-specify quantities related to the level of underlying smoothness.

The AdaGrad algorithm was originally proposed for composite convex optimization problems of the form 
\begin{equation}
    \min_{\bx\in\bHH} F(\bx) := f(\bx)+\varphi(\bx),
\label{eq:composite-problem}
\end{equation}
where \(f\colon\bHH\to\R\) is convex, and \(\varphi\colon\bHH\to\R\cup\{+\infty\}\) is proper, convex and lower semicontinuous. The stepsize is given by the accumulation of gradient information related to $f$, while a proximity operator deals with the presence of $\varphi$. The first theoretical analysis was provided when $f$ is assumed to be $G$-Lipschitz-continuous.
Later works considered the smooth setting and asked whether AdaGrad can adapt to the unknown smoothness level of the objective, while attaining the corresponding standard rate. For H\"older-smooth functions with exponent $\nu\in\left]0,1\right]$, \citep{orabona2023normalized} has proved that AdaGrad(Norm), the scalar variant in which all coordinates share the same accumulated gradient norm, achieves a deterministic convergence rate in the objective gap 
\begin{equation*}
    \mathcal O(n^{-(1+\nu)/2}).
\end{equation*}
However, this result addresses only the noncomposite setting, where the update reduces to a gradient step.
Whether the same convergence rates apply to the composite setting remains an open question.
In a composite problem, stationarity is not measured by the gradient of the smooth part alone. The first-order optimality condition is
\[
0\in \nabla f(\bx^\star)+\partial\varphi(\bx^\star),
\]
where \(\partial\varphi(\bx^\star)\) denotes the subdifferential of \(\varphi\) at \(\bx^\star\). This creates a mismatch for AdaGrad when its metric is driven only by the gradients of the smooth part: even if the iterates approach a composite minimizer, the accumulated gradient norms may keep growing linearly, as the gradient of $f$ does not vanish at a minimizer. Consequently, the effective stepsizes may decay as if the method were not approaching stationarity. As shown in Section~\ref{sec:adagrad-lower-bound}, this mismatch has a concrete effect on convergence: even when the differentiable term is H\"older-smooth, AdaGrad may converge arbitrarily close to the nonsmooth rate $\mathcal O (n^{-1/2})$. This failure of the AdaGrad scheme leads to the natural question. What quantity should be accumulated instead?

\paragraph{From Gradient Magnitude to Gradient Variation.}
\begin{wrapfigure}{r}{0.35\linewidth}
    \centering
    \includegraphics[width=\linewidth]{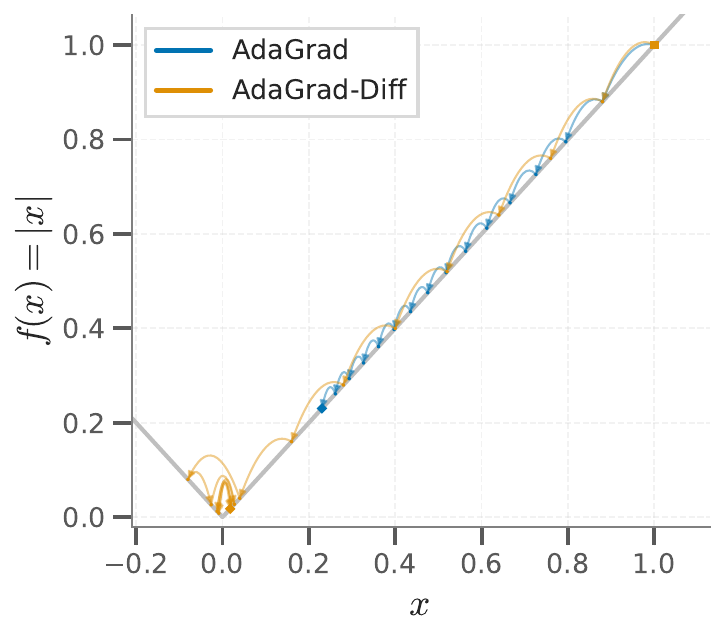}
    \caption{\small AdaGrad and AdaGrad-Diff on the objective $f(\x) =\abs{x}$ with $15$ iterations, $x^0=1$ and $\eta=0.12$.}
    \label{fig:abs-loss-regimes}
\end{wrapfigure}
This suggests that, in composite optimization, gradient magnitude is not always the right quantity for stepsize adaptation. 
On the other hand, AdaGrad interprets large gradients as evidence that future steps should be reduced. However, \emph{gradient magnitude alone} does not distinguish between two very different situations. 
In a \emph{steep and low curvature} region of the objective,
gradients may be large and yet remain well aligned, so aggressively shrinking the stepsize may be unnecessarily conservative. By contrast, in a \emph{high curvature} region, consecutive gradients may vary sharply in sign or direction, indicating local instability even when the instantaneous gradient is not particularly large. This suggests that an adaptive stepsize rule should not rely solely on the magnitude of the current gradient, but should also account for the stability of gradients across iterations. 

In this paper, motivated by the latter considerations, we study AdaGrad-Diff, a block-wise extension of the gradient-difference stepsize rule considered by \cite{rodomanov2024universality} within their UniSgd framework. Instead of scaling updates through the cumulative magnitude of past gradients, the proposed variant AdaGrad-Diff uses cumulative successive gradient differences, in a block-wise fashion. Although UniSgd is already known to attain the standard convergence rates for H\"older-smooth objectives, our analysis directly controls the growth of the accumulation in the stepsize. In the deterministic Lipschitz-smooth setting, we prove summability of the squared gradient differences, which enables us to establish convergence of the iterates via a variable-metric quasi-Fej\'er argument. More generally, we show that, in the deterministic H\"older-smooth setting, the stepsizes decay no faster than the classical schedule, thereby avoiding the excessively fast stepsize decay responsible for AdaGrad's failure in composite optimization.

\subsection{Related Work}
Adaptive gradient methods are a central class of first-order algorithms in
machine learning and signal processing. AdaGrad, introduced independently in
\citep{duchi2011adaptive} and \citep{streeter2010less} in the online setting,
is a foundational example. Its adaptive metric, built from cumulative squared
gradients, has since inspired a large family of algorithms. RMSProp
\citep{hinton} replaces AdaGrad's accumulation by an exponential moving
average in order to avoid rapid stepsize decay, AdaDelta
\citep{zeiler2012adadelta} further removes the parameter $\eta$,
and Adam \citep{kingma2014adam} combines RMSProp-type normalization with
momentum. The convergence issues of Adam in convex settings led to corrected
variants such as AMSGrad and related analyses
\citep{reddi2019convergence,alacaoglu2020new}.

A complementary line of work focused on whether AdaGrad can adapt to the
unknown smoothness level of convex objectives while attaining the corresponding
standard rate. An early result in this direction was given by
\cite{levy2018online}, who, to the best of our knowledge, were the first to
show that AdaGrad(Norm) can attain the standard smooth convergence rate. This
result, however, applies to constrained optimization over a bounded set and
assumes that a minimizer lies in the interior of the domain. 
More recently, in the unconstrained setting,
\citep{orabona2023normalized} showed that AdaGrad(Norm) can adapt to H\"older
smoothness, attaining the standard rate $\mathcal O(n^{-(1+\nu)/2})$, where $\nu\in\left]0,1\right]$ is the H\"older exponent. This behavior is sometimes described as “universality” --- we note, however, that universality in the sense of \citep{nesterov2015universal} refers to achieving the optimal accelerated rate uniformly over all \(\nu\).

Universality has also been studied for modified AdaGrad schemes in constrained and composite optimization. The variant considered by \cite{rodomanov2024universality} accumulates successive gradient differences rather than gradients themselves, recovering the standard H\"older convergence rates under a bounded domain assumption. The use of successive first-order differences is also at the basis of the works \citep{zhuang2020adabelief} and \citep{yue2023agd}, which
incorporate such mechanism through moving averages, and are therefore
closer in spirit to Adam.
Another approach replaces gradients in the AdaGrad stepsize with quantities tailored to constrained and composite problems. For constrained convex optimization, \citep{ene2021adaptive} proposed AdaGrad+, whose adaptive scaling is driven by gradient mappings and recovers the standard Lipschitz-smooth rate, provided that the domain is bounded. More recently, \citep{wang2026universal} extended this principle to general composite optimization.

These cumulative schemes are not the only route to adaptive or universal
first-order methods. A separate line of work uses local curvature information
obtained from successive iterate and gradient differences to adapt the
stepsize. In the noncomposite setting,
\citep{malitsky2020adaptive} introduced an adaptive gradient method requiring
only local Lipschitz continuity of the gradient, without prior knowledge of any smoothness constant. \citep{malitsky2024adaptive} subsequently extended
this approach to proximal-gradient methods for composite optimization under the
same local smoothness assumption. Closely related adaptive proximal schemes
were developed by \cite{latafat2025adaptive}, again under local Lipschitz
continuity of the smooth component. More recently,
\citep{oikonomidis2024adaptive} showed that this family of
adaptive proximal-gradient methods remains convergent under merely local
H\"older continuity of the gradient, without prior knowledge of either the
H\"older constant or exponent. This establishes a broader form of
universality for locally adaptive proximal-gradient methods, although the
convergence rate obtained there under H\"older-smoothness is $\mathcal O (n^{-\nu})$, which is weaker than the
standard $\mathcal O(n^{-(1+\nu)/2})$ rate. The theory of this line of methods remains predominantly deterministic: a recent stochastic extension by \cite{aujol2025stochastic} requires modifying the original stepsize rule, while the available quantitative guarantees rely on strong convexity. Beyond convergence rates, the above deterministic schemes also establish convergence of the iterates.

Finally, iterate convergence of vanilla AdaGrad for Lipschitz-smooth convex objectives was established in \citep{traore2021sequential} in a finite-dimensional setting, for both its scalar and coordinate-wise variants. Their analysis relies on variable-metric quasi-Fej\'er monotonicity and the descent lemma, exploiting the explicit relation between successive iterate differences and the gradient in the unconstrained setting. This argument does not directly extend to composite problems, where the update involves a proximity operator.

\subsection{Contributions}

In this section, we summarize the main contributions.

\begin{itemize}[leftmargin=3ex]
    \item We identify a failure mechanism of AdaGrad in composite optimization. More precisely, for every H\"older exponent $\nu\in\left]0,1\right]$, we construct a composite problem with a H\"older-smooth differentiable term for which the convergence rate of AdaGrad can be arbitrarily close to the nonsmooth rate $\mathcal{O}(n^{-1/2})$, rather than the standard rate $\mathcal{O}(n^{-(1+\nu)/2})$.

    \item In the deterministic case,  when $f$ is Lipschitz-smooth, we prove weak convergence of the iterates generated by AdaGrad-Diff to a minimizer in infinite-dimensional Hilbert spaces. 
    This result is based on controlling the gradient-difference accumulation, without requiring any prior boundedness assumption as in \citep{rodomanov2024universality}, ultimately making it possible to prove a quasi-Fej\'er property of the sequence. This complements the analysis given in the previously mentioned work.
    
    \item From the methodological point of view, we provide a convergence analysis different from that in \citep{rodomanov2024universality}. For general H\"older-smooth $f$, beyond convergence rates, our analysis allows to explicitly control the rate of decay of the stepsizes, thereby explaining why
    AdaGrad-Diff recovers the standard convergence rates while
    avoiding the excessively fast stepsize decay responsible for
    AdaGrad's failure in composite optimization. Our analysis also covers the stochastic setting considered in \citep{rodomanov2024universality}, with a slightly weaker interpolation condition.

    \item We evaluate AdaGrad-Diff on a range of optimization problems. The experiments show that it performs competitively with optimally tuned AdaGrad while remaining effective over a substantially broader range of values of $\eta$. In particular, AdaGrad-Diff compensates for excessively small values of $\eta$ by allowing more aggressive progress, while damping the effective stepsizes when $\eta$ is too large. This behavior improves stability and reduces sensitivity to hyperparameter tuning.
\end{itemize}

\subsection{Organization of the Paper}

The remainder of this section introduces the notation used throughout the paper. Section~\ref{sec:adagrad-lower-bound} develops the counterexample and states the lower bound for AdaGrad in composite optimization. In Section~\ref{2}, we present the algorithmic framework and state the main theoretical results for AdaGrad-Diff. Section~\ref{3} develops the convergence analysis and then derives sharper iterate-convergence guarantees in the deterministic Lipschitz-smooth case. Section~\ref{4} reports numerical experiments illustrating the performance of the proposed method. Some proofs and additional experimental results are deferred to the Appendix.
\subsection{Notation}

We denote by $\N$ the set of natural numbers and set
$\R_+ = \left[0,+\infty\right[$ and $\R_{++} = \left]0,+\infty\right[$.
For every $\m\in \N$, we also set $[\m]=\{1,\dots,\m\}$. If $\HH$
is a Hilbert space, we denote by $\scalarp{\cdot,\cdot}$ and
$\norm{\cdot}$ its scalar product and norm, respectively.
The Loewner ordering on bounded self-adjoint operators on $\HH$ is defined as
$\W_1\preceq\W_2 \Leftrightarrow (\forall\,\x\in\HH)\
\scalarp{\W_1\x,\x}\leq\scalarp{\W_2\x,\x}$. An operator $\W$ is said
to be positive if $\W\succeq0$ and bounded from below if
$\W\succeq\alpha\Id$ for some $\alpha>0$.
For a bounded positive self-adjoint operator
$\W\colon\HH\to\HH$ which is bounded from below, we set
$\scalarp{\x,\y}_{\W}:=\scalarp{\W\x,\y}$ and
$\norm{\x}_{\W}:=\scalarp{\W\x,\x}^{1/2}$. We regard $\HH$ endowed with
$\scalarp{\cdot,\cdot}_{\W}$ as a Hilbert space.
We denote by $\Gamma_0(\HH)$ the set of extended real-valued functions
$h\colon\HH\to\RX$ which are proper, convex and lower semicontinuous.
If $h\in\Gamma_0(\HH)$, we denote by $\argmin h$ the set of its minimizers
and, for every $\x\in\HH$, define its subdifferential in
$(\HH,\scalarp{\cdot,\cdot}_{\W})$ by
$\partial^{\W}h(\x):=\{\uu\in\HH\,\vert\,(\forall\,\y\in\HH)\
h(\y)\geq h(\x)+\scalarp{\y-\x,\uu}_{\W}\}$. Its proximity
operator in the metric induced by $\W$ is
$\prox_{h}^{\W}(\x):=\argmin_{\y\in\HH}
\{h(\y)+\frac12\norm{\y-\x}_{\W}^2\}$.
For $\W=\Id$, we omit the superscript $\W$ and write
$\partial h$, $\prox_{h}$. 
% We use the identity
% $\partial^{\W}h(\x)=\W^{-1}\partial h(\x)$.
If $\HH_1,\dots,\HH_{\m}$ are Hilbert spaces, their direct sum is the
vector space $\bHH=\prod_{i=1}^{\m}\HH_i$, endowed with the scalar product
$\scalarp{\bx,\by}=\sum_{i=1}^{\m}\scalarp{\x_i,\y_i}$, where
$\bx=(\x_i)_{1\leq i\leq\m}$ and
$\by=(\y_i)_{1\leq i\leq\m}$. This Hilbert space is denoted by
$\bigoplus_{i=1}^{\m}\HH_i$. If
$\w_1,\dots,\w_{\m}\in\R_{++}$ and $\Id_i$ denotes the identity operator
of $\HH_i$, we denote by
$\W=\bigoplus_{i=1}^{\m}\w_i\Id_i\colon\bHH\to\bHH$ the block-diagonal
operator $\bx\mapsto(\w_i\x_i)_{1\leq i\leq\m}$. 

\section{A lower Bound for AdaGrad in Composite Optimization}
\label{sec:adagrad-lower-bound}

We now make precise the mechanism described in the introduction. We construct a one-dimensional composite problem for which the gradient of the smooth component remains bounded away from zero and is accumulated by AdaGrad. As a result, the accumulation of gradient norms rules out the standard convergence rate of $\mathcal O(n^{-(1+\nu)/2})$.

\subsection{Construction of the Pathological Example}

Referring to problem~\eqref{eq:composite-problem}, we set $d=1$ and
\begin{equation}
\label{def_f_reg}
    f(x):=g_{p,c}(x)-x, \qquad \varphi(x):=x,
\end{equation}
where \(g_{p,c}\colon\R\to\R\) is defined, for \(p>2\), \(c\in\left]0,1\right[\), by
\begin{equation}
\label{eq:func_g}
g_{p,c}(x)
:=
\begin{cases}
\dfrac{c}{p}|x|^p, & |x|\le 1,\\[0.8em]
c\left(|x|-1+\dfrac1p\right), & |x|>1.
\end{cases}
\end{equation}
In particular, the composite objective is simply
\begin{equation}
\label{eq:20261007a}
 F=f+\varphi=g_{p,c},   
\end{equation}
which has the unique minimizer $x^\star=0$. The properties of the function $g_{p,c}$ are collected in the following proposition. We defer the proof to the Appendix.

\begin{proposition}\label{prop:prperties_g}
Let $p>2$ and $c\in \left]0,1\right[$.
The function $g_{p,c}$, defined in \eqref{eq:func_g}, is convex and differentiable with derivative
\begin{equation}
    \label{derivative}
    (\forall\,x\in \R)\quad
    g'_{p,c}(x)
    =
    \begin{cases}
    c\,\operatorname{sign}(x)|x|^{p-1}, & |x|\le 1,\\[0.4em]
    c\,\operatorname{sign}(x), & |x|>1.
    \end{cases}
\end{equation}
Moreover,
$g'_{p,c}$ is bounded and Lipschitz-continuous with constant
$L=c(p-1)$. Consequently, for every
$\nu\in \left]0,1\right]$,
$g'_{p,c}$ is $\nu$-H\"older continuous with constant
$c\max\{2,p-1\}$.
\end{proposition}

\begin{corollary}\label{cor:holder}
The function $f$ defined in \eqref{def_f_reg} is convex and Lipschitz-smooth with constant $L=c(p-1)$.
Moreover, for every $\nu\in\left]0,1\right]$,
$f^\prime$ is $\nu$-H\"older continuous with constant
$c\max\{2,p-1\}$.
\end{corollary}

\begin{remark}
The derivative of the smooth part of $F$ does not vanish at the minimizer:
\[
    f'(0)=g'_{p,c}(0)-1=-1\neq0.
\]
This is the key feature of the construction. The point
\(x^\star=0\) is a minimizer of the composite function, but the derivative of the smooth part \(f'(x^\star)\) is nonzero. Therefore an AdaGrad metric driven by $|f'(x^n)|^2$ may keep growing linearly even when $x^n\to x^\star$.
\end{remark}

\subsection{The AdaGrad Scheme and the Lower Bound}

Let $x^0\in\R$, $\eta>0$, and $\varepsilon>0$. In the one-dimensional setting considered above, the AdaGrad stepsize and the corresponding proximal update are given by\footnote{The general formulation of AdaGrad is that of \eqref{eq:prox-update-general}-\eqref{AdaGrad_stepsizes}.}
\begin{equation}
\label{eq:adagrad-lb-metric}
\gamma_n:=\eta/ w^n,\qquad
    w^n
    :=
    \varepsilon+
    \sqrt{\sum_{k=0}^n |f'(x^k)|^2},
\end{equation}
and
\begin{equation}
\label{eq:adagrad-lb-update}
    x^{n+1}
    =
    \prox_{\gamma_n\varphi}
    \left(
        x^n-\gamma_nf'(x^n)
    \right).
\end{equation}
Since $\varphi(x)=x$, it follows that
\begin{equation*}
(\forall\,\gamma>0)(\forall\,v\in\R)\qquad
    \prox_{\gamma\varphi}(v)=v-\gamma.
\end{equation*}
Recalling \eqref{def_f_reg} we have that $f'(x)=g_{p,c}^\prime(x)-1$, and hence the iteration
\eqref{eq:adagrad-lb-update} therefore reduces to
\begin{align}
\label{eq:20260818c}
x^{n+1} = x^n-\gamma_n f'(x^n)-\gamma_n
= x^n-\gamma_n g_{p,c}^\prime(x^n)
=x^n-\gamma_n F'(x^n).
\end{align}
Thus, the affine part of $f$ is exactly cancelled out by the proximal step associated with $\varphi$. However, it remains in the AdaGrad denominator, which is built from
$f'(x^n)=g_{p,c}'(x^n)-1$.
Indeed, by construction, $|g_{p,c}'(x)|\leq c<1$ for every $x\in\R$. Hence
\[
    1-c\leq |f'(x)|\leq 1+c.
\]
Then, \eqref{eq:adagrad-lb-metric} yields, for all
$n\in\N$, $\varepsilon+(1-c)\sqrt{n+1} \leq w^n
\leq \varepsilon+(1+c)\sqrt{n+1}$. Consequently,
\begin{equation}
\label{eq:stepsize}
(\forall\,n\in \N)\qquad
\frac{\eta}{\varepsilon+(1+c)\sqrt{n+1}}
\leq \gamma_n
\leq \frac{\eta}{\varepsilon+(1-c)\sqrt{n+1}},
\end{equation}
and the stepsizes decay as $n^{-1/2}$.

Before presenting the main theorem of the section, we collect a few properties of  $g_{p,c}^\prime$, which are a direct consequence of Proposition~\ref{prop:prperties_g}.
\begin{itemize}
\item $g'_{p,c}(x)$ is odd and $\abs{g'_{p,c}(x)}= g'_{p,c}(\abs{x})$.
\item $\abs{g'_{p,c}(x)} \leq L \abs{x}$ with $L = c(p-1)$
\item $\abs{g'_{p,c}(x)}\leq c \abs{x}^{p-1}$ and $\abs{g'_{p,c}(x)} \leq c$
\end{itemize}
The main result is given below.
\begin{theorem}[Lower Bound]
\label{lem:nonasymptotic-lower-bound}
Let $F=f+\varphi$, with $f$ and $\varphi$ defined in
\eqref{def_f_reg}, with $p>2$ and
$c\in\left]0,1\right[$. Let $\eta$ be such that
\begin{equation}
\label{eq:20260818e}
    0<\eta<\frac{1-c}{c(p-1)},
\end{equation}
and let $(x^n)_{n\in\N}$ be generated by
AdaGrad \eqref{eq:adagrad-lb-metric}-\eqref{eq:adagrad-lb-update}, with
$x^0\neq0$. Then there exists a constant $\kappa>0$, (depending on
$x^0,p,c$, and $\eta$) such that, for every $n\in \N$, with $n\geq 1$,
\begin{equation*}
    F(x^n)-F_\star
    \geq
    \kappa\,n^{-\frac{p}{2(p-2)}}.
\end{equation*}
\end{theorem}
\begin{proof}
We proceed by several steps. We set for the sake of brevity $g = g_{p,c}$ and $L= c(p-1)$.\\
\textbf{Step 1}.  Let $\iter \in \N$ and suppose that $x^\iter\neq 0$.
It follows from \eqref{eq:20260818c} and the properties of $g^\prime$ mentioned above, that
\begin{equation*}
x^{\iter+1}\cdot \mathrm{sign}(x^\iter) = \abs{x^\iter} - \gamma_\iter \abs{g^\prime(x^\iter)} \geq \abs{x^\iter} - \gamma_\iter L \abs{x^\iter}
= (1 - \gamma_\iter L) \abs{x^\iter}.
\end{equation*}
Moreover, by \eqref{eq:stepsize},
\begin{equation*}
\gamma_\iter L \leq \frac{\eta}{1 - c}\cdot c(p-1)\quad \text{and}\quad \frac{\eta}{1 - c}\cdot c(p-1) <1\ \Leftrightarrow\ \eta<\frac{1-c}{c(p-1)}.
\end{equation*}
Thus, since \eqref{eq:20260818e} holds, we have $\gamma_\iter L<1$ and hence
\begin{equation*}
x^{\iter+1}\cdot \mathrm{sign}(x^\iter) \geq (1 - \gamma_\iter L) \abs{x^\iter}> 0,
\end{equation*}
which in turn implies
\begin{equation*}
\mathrm{sign}(x^{\iter+1})=\mathrm{sign}(x^\iter)\quad\text{and}\quad \abs{x^{\iter+1}} = \abs{x^\iter} - \gamma_\iter \underbrace{g^\prime(\abs{x^\iter})}_{>0} > (1 - \gamma_\iter L) \abs{x^\iter}> 0.
\end{equation*}
In the end, the following hold. Let, for every $\iter\in \N$, $r_\iter = \abs{x^\iter}$. Then
%Let $\iter \in \N$ and suppose that $r_\iter>0$. Then
\begin{equation*}
(\forall\,\iter\in \N)\quad
r_\iter>0\ \Rightarrow\ 
r_{\iter+1} = r_{\iter} - \gamma_\iter g^\prime(r_\iter) > (1- \gamma_\iter L) r_\iter >0
\ \Rightarrow\ 0< r_{\iter+1}<r_{\iter}.
\end{equation*}
%Hence $0< r_{\iter+1}<r_{\iter}$.\\
\textbf{Step 2}. Suppose that $r_0>0$. It follows from the previous point that $(r_\iter)_{\iter\in \N}$ is decreasing. Thus $r_\iter \downarrow \bar{r} := \inf_{\iter} r_\iter$.
Since $g^\prime$ is increasing on $\R_+$, we have
\begin{equation*}
(\forall\,\iter\in \N)\quad
\gamma_\iter g^\prime(\bar{r})\leq \gamma_\iter g^\prime(r_\iter) = r_\iter - r_{\iter+1}.
\end{equation*}
Hence 
\begin{equation*}
g^\prime(\bar{r}) \sum_{\iter=0}^{+\infty} \gamma_\iter \leq r_0<+\infty \ \Rightarrow\ g^\prime(\bar{r})=0\ \Rightarrow\ \bar{r}=0,
\end{equation*}
where we used the fact that $\sum_{\iter=0}^{+\infty} \gamma_\iter=+\infty$. Therefore $r_\iter \to 0$.\\
\textbf{Step 3}. Since $r_{\iter+1} = r_{\iter} - \gamma_\iter g^\prime(r_\iter)$, we have
\begin{equation*}
\frac{r_{\iter+1}}{r_\iter} = 1 - \gamma_\iter \frac{g^\prime(r_\iter)}{r_\iter}
\quad\text{and}\quad \frac{g^\prime(r_\iter)}{r_\iter} \leq \frac{c r^{p-1}_\iter}{r_\iter} = c r_\iter^{p-2}\to 0.
\end{equation*}
Hence 
\begin{equation}
\label{eq:20260818f}
\frac{r_{\iter+1}}{r_\iter}\to 1.
\end{equation}
\textbf{Step 4}. Set $q = p-2 > 0$ and let $\iter \in \N$. By the mean value theorem we have
\begin{equation*}
\frac{1}{r_{\iter+1}^q} - \frac{1}{r_{\iter}^q} = \frac{r_\iter^q - r_{\iter+1}^q}{(r_{\iter} r_{\iter+1})^q} = \frac{q \theta_\iter^{q-1}(r_{\iter} - r_{\iter+1})}{(r_{\iter} r_{\iter+1})^q},\quad\text{with}\quad r_{\iter+1}<\theta_\iter<r_\iter.
\end{equation*}
Now we note that
\begin{equation*}
\theta_{\iter}^{q-1} = \frac{\theta_{\iter}^{q} }{\theta_{\iter}} \leq \frac{r_\iter^q}{r_{\iter+1}}.
\end{equation*}
Therefore
\begin{equation*}
\frac{1}{r_{\iter+1}^q} - \frac{1}{r_{\iter}^q}  \leq \frac{q r_\iter^q(r_{\iter} - r_{\iter+1})}{r_{\iter+1}(r_{\iter} r_{\iter+1})^q} =
\frac{q \gamma_\iter g^\prime(r_\iter)}{r_{\iter+1}^{p-1}}  \leq q \gamma_\iter c \underbrace{\bigg(\frac{r_\iter}{r_{\iter+1}}\bigg)^{p-1}}_{\substack{\downarrow \\1}}
\end{equation*}
In the end there exists $b>1$ such that
\begin{equation*}
(\forall\,\iter\in \N)\quad \frac{1}{r_{\iter+1}^q} - \frac{1}{r_{\iter}^q} \leq q c b \gamma_\iter
\end{equation*}
Then
\begin{equation*}
\sum_{\ite=0}^\iter \frac{1}{r_{\ite+1}^q} - \frac{1}{r_{\ite}^q}  \leq (p-2) c b \sum_{\ite=0}^\iter \gamma_{\ite}
\end{equation*}
and hence, recalling \eqref{eq:stepsize},
\begin{equation*}
\frac{1}{r_{\iter+1}^q} - \frac{1}{r_{0}^q} \leq \underbrace{(p-2) c b \frac{\eta}{1-c}}_{B/2} \sum_{\ite=0}^\iter \frac{1}{\sqrt{\ite+1}}.
\end{equation*}
Noting that $\sum_{\ite=1}^{\iter+1} 1/\sqrt{\ite} \leq 2 \sqrt{\iter+1}$, we have
\begin{equation*}
\frac{1}{r_{\iter+1}^q} \leq \frac{1}{r_{0}^q} + B \sqrt{\iter+1} \leq \big(r_{0}^{-q} + B \big) \sqrt{\iter+1}
\end{equation*}
and hence
\begin{equation}
\label{eq:20260818g}
(\forall\iter\in \N)\quad
r_{\iter+1} \geq \frac{1}{(r_0^{-q} + B)^{\frac 1 q}} (\iter+1)^{-\frac{1}{2(p-2)}}.
\end{equation}
\textbf{Step 5}. Let $\bar{\iter} = \min\{\iter \in \N\,\vert\, \abs{x^\iter} \leq 1\}$. Then, for every integer $\iter \geq \bar{\iter}$,
\begin{equation*}
\abs{x^\iter}\leq 1\ \Rightarrow\ g(\abs{x^\iter}) = \frac{c}{p}\abs{x^\iter}^p \geq \underbrace{\frac{c}{p} \frac{1}{(r_0^{-(p-2)} + B)^{\frac{p}{p-2}}}}_{C_1} \iter^{-\frac{p}{2(p-2)}}.
\end{equation*}
Set $C_2 = \min_{1 \leq \iter \leq \bar{\iter}-1} g(\abs{x^\iter})>0$. Then
\begin{equation*}
\forall\, \iter=1,\dots, \bar{\iter}-1\quad g(\abs{x^\iter}) \geq C_2 \geq C_2 n^{-\frac{p}{2(p-2)}}.
\end{equation*}
In the end, setting $\kappa = \min\{C_1, C_2\}$ we have that, for every $n\in\N$, with $n\geq1$
\begin{equation*}
\underbrace{g(x^\iter)}_{g(\abs{x^\iter})} - \underbrace{\min g}_{0} \geq \kappa \cdot n^{-\frac{p}{2(p-2)}}
\end{equation*}
and the statement follows. 
\end{proof}
We now show that on problem \eqref{eq:composite-problem}, with $f$ and $\varphi$ as in \eqref{def_f_reg}, AdaGrad, despite the $\nu$-H\"older-smoothness of $f$, may feature a convergence rate which is arbitrarily close to the nonsmooth $\mathcal O(n^{-1/2})$.

\begin{corollary}[Failure of AdaGrad for every H\"older-smooth Class]
\label{thm:main}
Fix $\nu\in\left]0,1\right]$ and $\alpha$ such that
\begin{equation*}
    \frac12<\alpha<\frac{1+\nu}{2}.
\end{equation*}
Then there exist a H\"older-smooth function $f$ with exponent
$\nu$,
and an affine function $\varphi$, such that, if $(x^n)_{n\in\N}$ is generated by AdaGrad, with
$x^0\neq0$ and sufficiently small $\eta>0$, then there exists
$\kappa>0$ such that, for every $n\in \N$, with $n\geq 1$,
\begin{equation*}
    F(x^n)-F_\star
    \geq
    \kappa n^{-\alpha}.
\end{equation*}
\end{corollary}

\begin{proof}
Set $p = 4\alpha /(2\alpha-1)>2$. Then
\begin{equation*}
\frac{p}{2(p-2)} = \alpha \in \left]\frac1 2, \frac{1+\nu}{2}\right[.
\end{equation*}
Choose $f$ and $\varphi$ as in Theorem~\ref{lem:nonasymptotic-lower-bound} (with $p$ given above and $c\in \left]0,1\right[$).
Then, by Corollary~\ref{cor:holder}, $f$ is $(\nu,L)$-H\"older-smooth, and the statement follows from Theorem~\ref{lem:nonasymptotic-lower-bound}. 
\end{proof}

\begin{remark}[Averaged iterate]
The same lower bound holds for the averaged iterate
\[
    \bar x^n:=\frac1n\sum_{k=1}^n x^k.
\]
Indeed, the sign of $x^n$ is preserved and $|x^n|$ is nonincreasing. Hence
$|\bar x^n|\geq|x^n|$. Since $F=g_{p,c}$ is even and nondecreasing as a function of $|x|$, we have
\[
    F(\bar x^n)-F_\star
    \geq
    F(x^n)-F_\star
    \geq
    \kappa n^{-\alpha}.
\]
\end{remark}

\section{AdaGrad-Diff and its Convergence}\label{2}
In this section we describe AdaGrad-Diff, and state the main theoretical guarantees. 
The proofs are postponed in the next section.
We address composite optimization problems of the form
\begin{equation}
    \min_{\bx\in \bHH}  F(\bx) := f(\bx) + \reg(\bx),
\label{eq:main_problem}
\end{equation}
under the following standing assumptions.
\begin{assumption}\label{ass:0}
$\bHH$ is a Hilbert space which is the direct sum of $\m\geq 1$ Hilbert spaces $\HH_1, \dots, \HH_\m$.
The function $f\colon \bHH\to \R$ is convex and the function $\reg\in\Gamma_0(\bHH)$. Moreover, $\argmin F\neq \varnothing$ and we set $F_\star = \min F$.
\end{assumption}
We consider two additional alternative regularity assumptions for $f$.

\begin{assumption}\label{ass:1}
$f$ is $G$-Lipschitz-continuous with $G>0$, i.e.,
\begin{equation*}
    \forall\,\bx,\by\in \bHH\colon \abs{f(\bx)-f(\by)}\leq G\norm{\bx-\by}.
\end{equation*}
\end{assumption}

\begin{assumption}\label{ass:2}
$f$ is H\"older-smooth with exponent $\nu\in\left]0,1\right]$ and constant $L > 0$, i.e., it is differentiable and
\begin{equation*}
    \forall\,\bx,\by\in \bHH\colon \norm{\nabla f(\bx)-\nabla f(\by)}\leq L\norm{\bx-\by}^\nu.
\end{equation*}
\end{assumption}
We assume to have access to an unbiased stochastic first-order oracle $\hat g$ satisfying the following variance condition.

\begin{assumption}
\label{ass:unbiased} $\mathcal Z$ is a measurable space and the function $\hat g\colon\bHH\times\mathcal Z\to\bHH$ is measurable. Moreover, $\zeta$ is a random variable with values in $\mathcal Z$ and, for every $\bx \in \bHH$, $\norm{\hat{g}(\bx, \zeta)}$ is square integrable and 
\[
    \forall\,\bx\in \bHH\colon g(\bx):=\mathbb E[\hat g(\bx,\zeta)]\in\partial f(\bx).
\]
Moreover, there exist
    $\bx^\star\in\argmin F$, $\C,\sigma\geq0$, and $q\in [0,1]$, such that\footnote{For \(q=1\), this is precisely the variance condition used in
\citep{juditsky2023sparse,ilandarideva2025accelerated}.
It is also consistent with Lemma~2.4 of \citep{gower2019sgd}, whose
second-moment bound yields the corresponding variance control in the
unconstrained setting, where the objective gap coincides with
\(D_f(\bx,\bx^\star)\). For \(q\in(0,1)\), \eqref{eq:state-dependent-variance}
is its H\"older counterpart.} 
    \begin{equation}
    \label{eq:state-dependent-variance}
        \forall\,\bx\in \bHH\colon\mathbb E \left[\|\hat g(\bx,\zeta)-g(\bx)\|^2\right]
        \leq \C[D_f(\bx,\bx^\star)]^q+\sigma^2,
    \end{equation}
    where 
    $D_f(\bx,\by)
        :=f(\bx)-f(\by)
        -\langle g(\by),\bx-\by\rangle$.
\end{assumption}

\begin{remark}
    Condition~\eqref{eq:state-dependent-variance} recovers uniformly bounded variance when $\C=0$. At the opposite extreme, $\sigma=0$ corresponds to an interpolation regime whenever $f$ is differentiable. Indeed, at the solution $\bx^\star$ we have  $\hat g(\bx^\star,\zeta)=\nabla f(\bx^\star)$ almost surely, and hence $0\in \hat g(\bx^\star,\zeta)+\partial\varphi(\bx^\star)$ almost surely.
\end{remark}
We now introduce the algorithm. We consider a variable metric setting that encompasses the coordinate-wise scaling commonly used in practice, where the norms $\norm{\,\cdot\,}_{W_\iter}$ are defined via a sequence $(\W_\iter)_{\iter\in \N}$ of positive definite block-diagonal bounded linear operators
\begin{equation*}
\W_\iter\colon \bHH\to \bHH\colon \bx\mapsto
\begin{bmatrix}
\w_1^n \x_1\\
\vdots\\
\w_\m^n \x_\m\\
\end{bmatrix}
\qquad
\|\bx\|_{\W_\iter} = \sqrt{\scalarp{\W_\iter \bx, \bx}} = \sqrt{\sum_{i=1}^\m \w_i^{\iter} \norm{x_i}^2},
\end{equation*}
where $\w_i^\iter>0$ are properly set. For a parameter $\eta_\iter>0$, the proximal update in the geometry induced by $\W_\iter$ is
\begin{align}\label{eq:prox-update-general}
   \bx^{\iter+1} &=\prox^{\W_\iter}_{\eta_\iter \reg}(\bx^\iter - \eta_\iter \W_\iter^{-1}\buu^\iter)\notag\\
    &= \argmin_{\bx}\left\{\scalarp{\buu^\iter, \bx-\bx^\iter} + \reg(\bx) + \frac{\norm{\bx-\bx^\iter}_{\W_\iter}^2}{2\eta_\iter} \right\},
\end{align}
where we set $\buu^\iter := \hat g(\bx^\iter,\zeta_\iter)$, $\zeta_\iter$ being a fresh realization of $\zeta$ at time $\iter$. 

A classical choice for the diagonal geometry is the \emph{AdaGrad metric}
\begin{equation}
\label{AdaGrad_stepsizes}
(\forall\,i\in [\m])\qquad \w_{i}^\iter
    := \varepsilon+\sqrt{\sum_{\ite=0}^\iter \norm{\uu^\ite_{i}}^2},\qquad \eta_\iter \equiv \eta,\qquad \varepsilon > 0,
\end{equation}
where $\buu^\ite=(\uu_i^\ite)_{1\leq i\leq\m}$. The key departure from AdaGrad is that the accumulator is instead driven by squared norms of successive stochastic-gradient differences:
\begin{equation}\label{AdaDiff}
(\forall\,i\in [\m])\qquad \w_{i}^{\iter} := \varepsilon+\sqrt{\sum_{\ite=0}^\iter \norm{\uu_{i}^\ite - \uu_{i}^{\ite-1}}^2}\;,
\end{equation}
where we use the convention $\buu^{-1}=0$, since no previous gradient is available at the first iteration.
This construction directly encodes our key insight:
when subgradients remain stable across iterations, the accumulated differences grow slowly and the
stepsize does not shrink unnecessarily. On the contrary, when subgradients fluctuate significantly, the
accumulated differences grow rapidly and the stepsize reduces to stabilize the optimization. 

\begin{algorithm}[AdaGrad-Diff]
\label{adadiff}
Let $\bx^0\in\bHH$, $\eta>0$, and $\varepsilon>0$. Set $\buu^{-1}=0\in\bHH$ and define
\begin{equation}
\label{eq:adagrad-diff}
\begin{array}{l}
\text{for}\; n=0,1,\ldots\\
\left\lfloor
\begin{array}{l}
\buu^n=\hat g(\bx^n,\zeta_n)\\[1.5ex]
(\forall\,i\in[\m])\;
w_i^n
=
\varepsilon+
\left(
\sum_{k=0}^n
\norm{\uu_i^k-\uu_i^{k-1}}^2
\right)^{1/2}\\[1.5ex]
\W_n=\bigoplus_{i=1}^{\m}w_i^n\Id_i\\[1.5ex]
\bx^{n+1}
=
\prox_{\eta\reg}^{\W_n}
\left(
\bx^n-\eta\W_n^{-1}\buu^n
\right).
\end{array}
\right.\\
\end{array}
\end{equation}
\end{algorithm}
We now present the main results. 

\begin{theorem}[Lipschitz-continuous Case]\label{thm_nonsmooth}
Let Assumptions~\ref{ass:0}, \ref{ass:1}, and \ref{ass:unbiased} with $\C=0$,
hold, and let $(\bx^n)_{n\in\N}$ be generated by Algorithm~\ref{adadiff}. 
Let $\bx^\star\in\argmin F$, 
$\iter\in \N$, with $n\geq 1$, and $\Delta^{\iter}_{\max}(\bx^\star) = \max_{0\leq \ite\leq \iter} \max_{1\leq i\leq \m} \norm{\x^\ite_i - x^\star_i}^2 $.
Define $\bar\bx^n=n^{-1}\sum_{k=1}^n\bx^k$ and
\begin{equation*}
R_n(\eta):=
\left(
\mathbb E\left[
\left(
\eta+\frac{5\Delta_{\max}^n(\bx^\star)}{2\eta}
\right)^2
\right]
\right)^{1/2} \in [0,+\infty],
\qquad w^0_{\max} = \max_{1\leq i\leq d} w_i^0.
\end{equation*}
Then
\begin{equation}\label{eq:explicit-nonsmooth}
 \mathbb E[F(\bar\bx^n)-F_\star]
 \leq \frac{\mathbb E[w^0_{\max}]}{2\eta n}\norm{\bx^0 - \bx^\star}^2 +\frac{\sqrt{\m}R_{\iter}(\eta)}{\sqrt n} 2\sqrt{G^2+\sigma^2}.
\end{equation}
In particular, if
$\sup_{n\in\N}\mathbb E[(\Delta_{\max}^n(\bx^\star))^2]<+\infty$\footnote{This holds, in particular, if $\mathrm{dom}\,\reg$ is bounded.}, then the right-hand side is $\mathcal O(n^{-1/2})$.
\end{theorem}

\begin{theorem}[Adaptive Accumulation Bound]\label{prop:accumulator_control}
Let Assumptions~\ref{ass:0} and~\ref{ass:2} hold with
\(\nu\in\left]0,1\right]\), and let Assumptions~\ref{ass:unbiased}
hold with
\(q:=2\nu/(1+\nu)\) and $\bx^\star\in \argmin F$.
Let $(\bx^n)_{n\in\N}$ be generated by
Algorithm~\ref{adadiff}. Let $\iter \in \N$, with $n\geq 1$, and let $\Delta^{\iter}_{\max}(\bx^\star) = \max_{0\leq \ite\leq \iter} \max_{1\leq i\leq \m} \norm{\x^\ite_i - x^\star_i}^2$. Then, the following hold.
\begin{enumerate}[label={\rm (\roman*)}, leftmargin=4ex, itemsep=1ex]
\item\label{proposition5_i} ${\sum_{\ite=1}^\iter \mathbb E\left[\lVert \buu^{\ite} - \buu^{\ite-1} \rVert^2\right]}=\mathcal{O}\left((\mathbb E [\Delta^\iter_{\max}(\bx^\star)])^{\nu} n^{1-\nu} + \sigma^2 \iter\right)$.
    \item\label{proposition5_ii} If the oracle is deterministic, then ${\sum_{\ite=1}^\iter \lVert \buu^{\ite} - \buu^{\ite-1} \rVert^2}=\mathcal{O}\left(\iter^{1-\nu}\right)$.
\end{enumerate}
\end{theorem}

\begin{remark}
    In the deterministic setting, by Theorem~\ref{prop:accumulator_control}\ref{proposition5_ii} and \eqref{AdaDiff}, we have
    %or more generally in the interpolation regime ($\sigma = 0$), 
    \begin{equation*}
    \frac{\eta}{w_i^n} = \Omega(n^{-(1-\nu)/2}).
    \end{equation*}
    This rate
     is exactly the decay required for the proximal gradient method to retain the classical convergence rate in the H\"older-smooth setting.
\end{remark}

\begin{theorem}[H\"older-smooth Case]\label{thm_smooth}
Let Assumptions~\ref{ass:0} and~\ref{ass:2} hold with
\(\nu\in\left]0,1\right]\), and let Assumptions~\ref{ass:unbiased}
hold with
\(q:=2\nu/(1+\nu)\)
and $\bx^\star\in\argmin F$. Let $(\bx^n)_{n\in\N}$ be generated by
Algorithm~\ref{adadiff}. Let $\iter \in \N$, with $n\geq 1$, and let $\Delta^{\iter}_{\max}(\bx^\star) = \max_{0\leq \ite\leq \iter} \max_{1\leq i\leq \m} \norm{\x^\ite_i - x^\star_i}^2$. 
Define $\bar\bx^n=n^{-1}\sum_{k=1}^n\bx^k$. Then
\begin{align}
\mathbb E[F(\bar\bx^n)-F_\star]
&\leq
\frac{\mathbb E[w^0_{\max}]}{2\eta n}
\|\bx^0-\bx^\star\|^2 \notag \\
& \qquad+
\frac{\sqrt{\m}R_n(\eta)}{n}
\bigg[
C_{0,\nu} +
C_{1,\nu}n^{\frac{1-\nu}{2(1+\nu)}}
+
C_{2,\nu}R_n(\eta)^\nu n^{\frac{1-\nu}{2}}
+
\sigma\sqrt{3 n}
\bigg],
\label{eq:explicit-holder}
\end{align}
where $R_\iter(\eta)$ and $\w^0_{\max}$ as in Theorem~\ref{thm_nonsmooth} and $C_{0,\nu},C_{1,\nu},C_{2,\nu}>0$ are independent of $n$
and are given explicitly in the subsequent proposition~\ref{proposition5}. In particular, if
$\sup_{n\in\N}\mathbb E[(\Delta_{\max}^n(\bx^\star))^2]<+\infty$,
then
\begin{equation*}
\mathbb E[F(\bar\bx^n)-F_\star]
=
\mathcal O\left(
n^{-\frac{1+\nu}{2}}
+
\sigma n^{-\frac{1}{2}}
\right).
\end{equation*}
\end{theorem}

\begin{remark}\ 
\begin{enumerate}[label={\rm (\roman*)}, leftmargin=4ex, itemsep=1ex]
\item Theorem~\ref{thm_nonsmooth} and Theorem~\ref{thm_smooth}
 extend the analogous results obtained in \citep{rodomanov2024universality} to AdaGrad-Diff, which is a block-wise version of UniSgd considered there.
\item Similarly to 
\citep{rodomanov2024universality,ene2021adaptive,wang2026universal},  
the above convergence rates 
require the boundedness of the iterates.
In the following, we show that, in the deterministic Lipschitz-smooth case, this assumption can be dropped and that, even more notably, the iterates generated by AdaGrad-Diff converge weakly to a minimizer. This extends the results in \cite{traore2021sequential} to the composite setting and to infinite-dimensional Hilbert spaces.
\end{enumerate}
\end{remark}

\begin{theorem}\label{thm_smooth_det}
Let Assumptions~\ref{ass:0} and~\ref{ass:2} hold with $\nu=1$, and suppose that the oracle is deterministic, i.e.,
$\hat g(\bx,\zeta)=\nabla f(\bx)$ almost surely for every $\bx$. Let $(\bx^n)_{n\in\N}$ be the sequence generated by Algorithm~\ref{adadiff}. Then, the following rate of convergence holds
\begin{equation}
\label{eq:rate2}
F(\bar\bx^n)-F_\star=\mathcal O(n^{-1}),
\end{equation}
and $\bx^n\rightharpoonup\bx^\star$ for some $\bx^\star\in\argmin F$.
\end{theorem}

\section{Convergence Analysis}\label{3}
\label{sec:analysis}
In this section, we provide the main arguments used to prove the theorems stated in the previous section. Additional details can be found in Appendix \ref{A}. 
In what follows 
\emph{we make Assumptions~\ref{ass:0} and~\ref{ass:unbiased} standing}
and let the sequence $(\bx^\iter)_{\iter\in\N}$ be generated by Algorithm~\ref{adadiff}.
Moreover, we set
\begin{equation*}
    (\forall\,\iter\in \N)(\forall\,i\in \{1,\dots, \m\})\quad
    v_i^\iter:= \sqrt{\sum_{\ite=0}^\iter\norm{\uu_i^\ite- \uu_i^{\ite-1}}^2},
\end{equation*}
so that, $\w_i^\iter= \varepsilon + v_i^\iter$. We note that, since $\uu_i^{-1}=0$, we have, for every $n\in \N$,
\begin{equation}
\label{eq:20260818b}
v_i^\iter  - v_i^0 = \sqrt{\norm{\uu_i^0}^2+\sum_{\ite=1}^\iter\norm{\uu_i^\ite- \uu_i^{\ite-1}}^2} - \norm{\uu_i^0}
\leq \sqrt{\sum_{\ite=1}^\iter\norm{\uu_i^\ite- \uu_i^{\ite-1}}^2}.
\end{equation}
We begin with a basic inequality for the proximal gradient step. The proof is deferred to the Appendix.

\begin{lemma}
\label{lem:basicPGstep}
Let $(X, \scalarp{\,\cdot\,, \,\cdot\,})$ be a real Hilbert space, let $f,\reg\in\Gamma_0(X)$, and set $F=f+\reg$. Let $x\in X$, let $u\in X$ , choose $v\in\partial f(x)$, let $\eta>0$, and set
\begin{equation*}
    x^+=\prox_{\eta\reg}(x-\eta u).
\end{equation*}
Then, for every $z\in X$,
\begin{multline}
\label{eq:basicPGstep}
2\eta(F(x^+)-F(z))
\leq \|x-z\|^2-\|x^+-z\|^2-\|x^+-x\|^2\\
+2\eta\big[f(x^+)-f(x)-\pair{x^+-x}{u}\big]
+2\eta\pair{x-z}{v-u}.
\end{multline}
\end{lemma}
From this, we derive two basic inequalities, stated in the next lemma.

\begin{lemma}[{\bf The Basic Inequalities}]\label{lemma1}
Let $n\in\N$ and define
\begin{equation*}
    \rho_n :=\min\left\{\eta^2 \norm{\buu^{n+1} - \buu^{n}}_{\W_{n}^{-1}}^2\;,\;\norm{\bx^{n+1} - \bx^n}_{\W_{n+1}-\W_n}^2
    + \eta^2 \norm{\buu^{n+1} - \buu^{n}}_{\W_{n+1}^{-1}}^2\right\}.
\end{equation*}
Then, for every $\bx\in \bHH$
\begin{equation}\label{lemma1_bound}
    2\eta\mathbb E[F(\bx^{n+1}) - F(\bx)]
    \le \mathbb E\left[\norm{\bx^n - \bx}_{\W_n}^2- \norm{\bx^{n+1} - \bx}_{\W_n}^2 +\rho_n\right].
\end{equation}
Moreover, under Assumption~\ref{ass:2}, it holds
\begin{align}\label{lemma1_bound2}
    2\eta\mathbb E[F(\bx^{n+1}) - F(\bx)]
    &\le \mathbb E\left[\norm{\bx^n - \bx}_{\W_n}^2- \norm{\bx^{n+1} - \bx}_{\W_n}^2 +\rho_n\right]\notag\\
    &\quad- \frac{2\nu}{1+\nu}\frac{\eta}{L^{1/\nu}}
\mathbb E[\|\nabla f(\bx^{n+1})-\nabla f(\bx^n)\|^{\frac{1+\nu}{\nu}}].
\end{align}
\end{lemma}

\begin{proof}
Let $\mathcal{F}_n:=\sigma(\zeta_0,\ldots,\zeta_{n-1})$ denote the filtration generated by the oracle samples up to iteration $n-1$.
We note that
\begin{equation*}
\bx^{\iter+1} = \prox^{\W_\iter}_{\eta \reg} (\bx^\iter- \eta \W_\iter^{-1}\buu^\iter). 
\end{equation*}
Thus, applying Lemma~\ref{lem:basicPGstep} 
in the Hilbert space $(\bHH, \scalarp{\,\cdot\,, \,\cdot\,}_{W_\iter})$, with $u=\W_n^{-1}\buu^n$ and $v=\W_n^{-1}g(\bx^n)\in\partial^{\W_n}f(\bx^n)$, 
we have,
\begin{align}
\label{eq:20260124d}
\nonumber 2\eta (F(\bx^{\iter+1})- F(\bx)) &\leq \norm{\bx^\iter\!-\bx}_{\W_\iter}^2\! - \norm{\bx^{\iter+1}-\bx}_{\W_\iter}^2 \!- \norm{\bx^{\iter+1}\!-\bx^\iter}_{\W_\iter}^2\\[1ex] 
\nonumber&\qquad+2\eta \big[f(\bx^{\iter+1})-f(\bx^\iter)- \pair{\bx^{\iter+1}\!-\bx^\iter}{\buu^\iter}\big]\\[1ex]
&\qquad+2\eta\pair{\bx^\iter-\bx}{g(\bx^\iter)-\buu^\iter}.
\end{align}
Now, we rewrite the term between brackets in \eqref{eq:20260124d} as follows
\begin{align*}
f(\bx^{\iter+1})&-f(\bx^\iter)- \pair{\bx^{\iter+1}-\bx^\iter}{\buu^\iter}\\
&= \pair{\bx^{\iter+1}-\bx^\iter}{g(\bx^{\iter+1})-\buu^\iter} - D_f(\bx^\iter, \bx^{\iter+1}) \\
&= \pair{\bx^{\iter+1}-\bx^\iter}{\buu^{\iter+1}-\buu^\iter}
 - D_f(\bx^\iter, \bx^{\iter+1})\\
& \qquad +\pair{\bx^{\iter+1}-\bx^\iter}{g(\bx^{\iter+1})-\buu^{\iter+1}}.
\end{align*}
Therefore, substituting the above expression into \eqref{eq:20260124d}, we obtain
\begin{align}
\label{eq:20261028a}
\nonumber 2\eta (F(\bx^{\iter+1})- F(\bx)) &\leq \norm{\bx^\iter\!-\bx}_{\W_\iter}^2\! - \norm{\bx^{\iter+1}-\bx}_{\W_\iter}^2 \!- \norm{\bx^{\iter+1}\!-\bx^\iter}_{\W_\iter}^2\\[1ex] 
\nonumber&\qquad
+2\eta \pair{\bx^{\iter+1}-\bx^\iter}{\buu^{\iter+1}-\buu^\iter}
 - 2\eta D_f(\bx^\iter, \bx^{\iter+1})\\[1ex]
 & \qquad +2\eta\pair{\bx^{\iter+1}-\bx^\iter}{g(\bx^{\iter+1})-\buu^{\iter+1}}\notag\\[1ex]
&\qquad+2\eta\pair{\bx^\iter-\bx}{g(\bx^\iter)-\buu^\iter}.
\end{align}
We now consider the quantity
\begin{equation*}
A := 2\eta \pair{\bx^{\iter+1}- \bx^\iter}{\buu^{\iter+1}- \buu^\iter} - \norm{\bx^{\iter+1}-\bx^\iter}_{\W_\iter}^2.
\end{equation*}
From one hand, using Young-Fenchel inequality, we have
\begin{align}
\label{eq:20260814b}
\nonumber A &= 2 \bigg[\pair{\bx^{\iter+1}-\bx^\iter}{\eta \W_\iter^{-1}(\buu^{\iter+1}-\buu^\iter)}_{\W_\iter} - \frac 1 2 \norm{\bx^{\iter+1}-\bx^\iter}_{\W_\iter}^2 \bigg]\\[1ex]
&\leq  \norm{\eta \W_\iter^{-1}( \buu^{\iter+1}-\buu^\iter)}^2_{\W_\iter}=  \eta^2 \norm{ \buu^{\iter+1}-\buu^\iter}^2_{\W^{-1}_\iter}.
\end{align}
On the other hand, summing and subtracting $\norm{\bx^{\iter} -\bx^{\iter+1}}_{\W_{\iter+1}}^2$ and using Young-Fenchel inequality as before, we get
\begin{align}
\label{eq:20260814c}
\nonumber A &=  \norm{\bx^{\iter+1}-\bx^\iter}_{\W_{\iter+1}}^2- \norm{\bx^{\iter+1}-\bx^\iter}_{\W_\iter}^2\\[1ex]
\nonumber & \qquad +2\eta \pair{\bx^{\iter+1}-\bx^\iter}{ \buu^{\iter+1}-\buu^\iter} - \norm{\bx^{\iter+1}-\bx^\iter}_{\W_{\iter+1}}^2\\[1ex]
& \leq \norm{\bx^{\iter+1}-\bx^\iter}_{\W_{\iter+1}-\W_{\iter}}^2 + 
\eta^2 \norm{\buu^{\iter+1}-\buu^\iter}^2_{\W^{-1}_{\iter+1}}.
\end{align}
Combining \eqref{eq:20260814b} and \eqref{eq:20260814c}, we have $A \leq \rho_\iter$. Therefore, from \eqref{eq:20261028a},
\begin{align}
\label{eq:basic-with-bregman}
\nonumber 2&\eta(F(\bx^{\iter+1})-F(\bx))\notag\\[1ex]
&\leq \norm{\bx^\iter-\bx}_{\W_\iter}^2-\norm{\bx^{\iter+1}-\bx}_{\W_\iter}^2+\rho_\iter
-2\eta 
D_f(\bx^\iter, \bx^{\iter+1})\notag\\[1ex]
&\quad+2\eta\pair{\bx^{\iter+1}-\bx^\iter}{g(\bx^{\iter+1})-\buu^{\iter+1}}
+2\eta\pair{\bx^\iter-\bx}{g(\bx^\iter)-\buu^\iter}.
\end{align}
The two inner products have zero expectation. Indeed,
$\bx^\iter-\bx$ is $\mathcal F_\iter$-measurable and
$\mathbb E[\buu^\iter-g(\bx^\iter)\mid\mathcal F_\iter]=0$, while
$\bx^{\iter+1}-\bx^\iter$ is $\mathcal F_{\iter+1}$-measurable and
$\mathbb E[\buu^{\iter+1}-g(\bx^{\iter+1})\mid\mathcal F_{\iter+1}]=0$.
Taking expectations in \eqref{eq:basic-with-bregman} gives
\begin{align}
\label{eq:basic-with-bregman-exp}
\nonumber 2\eta&\mathbb E[F(\bx^{\iter+1})-F(\bx)]\notag\\[1ex]
&\leq \mathbb E\left[\norm{\bx^\iter-\bx}_{\W_\iter}^2-\norm{\bx^{\iter+1}-\bx}_{\W_\iter}^2+\rho_\iter\right] -2\eta\mathbb E[D_f(\bx^\iter, \bx^{\iter+1})]. 
\end{align}
Since $D_f(\bx^\iter,\bx^{\iter+1})\geq0$, \eqref{lemma1_bound} follows.

As for the second part of the statement, suppose that Assumption~\ref{ass:2} holds. By the H\"older cocoercivity inequality
\cite[Corollary 3.5.7 (iii)]{zalinescu2002convex},
\begin{equation*}
D_f(\bx^\iter,\bx^{\iter+1})
\geq
\frac{\nu}{1+\nu}\frac{1}{L^{1/\nu}}
\norm{\nabla f(\bx^{\iter+1})-\nabla f(\bx^\iter)}^{\frac{1+\nu}{\nu}}.
\end{equation*}
Substituting this estimate into \eqref{eq:basic-with-bregman-exp} gives
\eqref{lemma1_bound2} and
the statement follows. 
\end{proof}
A natural way to turn Lemma~\ref{lemma1} into a convergence bound is to sum its first inequality over the iterations. The distance terms then exhibit an almost telescopic structure. Among the two bounds entering $\rho_n$, the second is particularly convenient, as each stochastic-gradient difference is scaled using an accumulator that already contains it, allowing a standard summation estimate to be applied. The additional distance term can be treated together with the resulting non-telescoping contributions. This leads to the following. 

\begin{proposition}\label{Proposition1}
Let $\iter \in \N$, with $n\geq 1$, and $\bx\in\bHH$.
Set $w_{\max}^0 = \max_{1\leq i\leq d} w_i^0$ and $\Delta^{\iter}_{\max}(\bx) = \max_{0\leq \ite\leq \iter} \max_{1\leq i\leq \m} \norm{\x^\ite_i - x_i}^2 $.
Then, 
\begin{align}
\sum_{k=1}^n
\eta\,
\mathbb E[F(\bx^k)-F(\bx)]
&\leq
\frac{\mathbb E[w_{\max}^0]}{2}
\norm{\bx^0-\bx}^2\notag\\
&\qquad+
\frac{5}{2}
\sum_{i=1}^{\m}
\mathbb E\left[
\Delta_{\max}^{n}(\bx)
\sqrt{
\sum_{k=1}^n
\norm{\uu_i^k-\uu_i^{k-1}}^2
}
\right]
\notag\\
&\qquad+
\frac{1}{2}
\sum_{k=1}^n
\eta^2\,
\mathbb E\left[
\norm{\buu^k-\buu^{k-1}}_{\W_k^{-1}}^2
\right].
\label{Proposition1_bound}
\end{align}
\end{proposition}

\begin{proof}
It follows from equation~\eqref{lemma1_bound} in Lemma~\ref{lemma1}
 that
\begin{align}
2\sum_{k=0}^{n-1}\eta
&\mathbb E\!\left[F(\bx^{k+1})-F(\bx)\right]\notag\\
&\leq
\sum_{k=0}^{n-1}
\mathbb E\!\left[
\|\bx^k-\bx\|_{\W_k}^2
-
\|\bx^{k+1}-\bx\|_{\W_k}^2
\right]
+
\sum_{k=0}^{n-1}\mathbb E[\rho_k].
\label{eq:prop1-stoch-start}
\end{align}
The first sum on the right-hand side can be bounded as follows:
\begin{align}
&\sum_{k=0}^{n-1}
\mathbb E\!\left[
\|\bx^k-\bx\|_{\W_k}^2
-
\|\bx^{k+1}-\bx\|_{\W_k}^2
\right]
\notag\\
&=
\mathbb E \left[\|\bx^0-\bx\|_{\W_0}^2\right]
+
\sum_{k=1}^{n-1}
\mathbb E \left[\|\bx^k-\bx\|_{\W_k}^2\right]
-
\sum_{k=0}^{n-2}
\mathbb E \left[\|\bx^{k+1}-\bx\|_{\W_k}^2\right]
-
\mathbb E \left[\|\bx^n-\bx\|_{\W_{n-1}}^2\right]
\notag\\
&\leq
\mathbb E \left[\|\bx^0-\bx\|_{\W_0}^2\right]
+
\sum_{k=0}^{n-2}
\mathbb E
\|\bx^{k+1}-\bx\|_{\W_{k+1}-\W_k}^2
\notag\\
&=
\varepsilon\|\bx^0-\bx\|^2
+
\sum_{i=1}^{\m}
\mathbb E\!\left[
v_i^0\|x_i^0-x_i\|^2
\right]+
\sum_{k=0}^{n-2}\sum_{i=1}^{\m}
\mathbb E\!\left[
(v_i^{k+1}-v_i^k)
\|x_i^{k+1}-x_i\|^2
\right]
\notag\\
&\leq
\mathbb E\!\left[w^0_{\max}\right]
\|\bx^0-\bx\|^2
+
\sum_{i=1}^{\m}
\mathbb E\!\left[
\Delta_{\max}^n(\bx)
(v_i^{n-1}-v_i^0)
\right].
\label{eq:prop1-stoch-metric}
\end{align}
On the other hand, recalling the definition of $\rho_k$ in
Lemma~\ref{lemma1}, the last sum on the right-hand side of
\eqref{eq:prop1-stoch-start} satisfies
\begin{align}
\sum_{k=0}^{n-1}\mathbb E[\rho_k]
&\leq
\sum_{k=0}^{n-1}
\mathbb E
\|\bx^{k+1}-\bx^k\|_{\W_{k+1}-\W_k}^2
+
\sum_{k=1}^{n}
\eta^2
\mathbb E
\|\buu^k-\buu^{k-1}\|_{\W_k^{-1}}^2
\notag\\
&=
\sum_{k=0}^{n-1}\sum_{i=1}^{\m}
\mathbb E\!\left[
(w_i^{k+1}-w_i^k)
\|x_i^{k+1}-x_i^k\|^2
\right]+
\sum_{k=1}^{n}
\eta^2
\mathbb E
\|\buu^k-\buu^{k-1}\|_{\W_k^{-1}}^2
\notag\\
&\leq
4\sum_{i=1}^{\m}
\mathbb E\!\left[
\Delta_{\max}^n(\bx)
(v_i^n-v_i^0)
\right]
+
\sum_{k=1}^{n}
\eta^2
\mathbb E
\|\buu^k-\buu^{k-1}\|_{\W_k^{-1}}^2,
\label{eq:prop1-stoch-rho}
\end{align}
where in the last inequality we used the majorization
\[
\|x_i^{k+1}-x_i^k\|^2
\leq
2\|x_i^{k+1}-x_i\|^2
+
2\|x_i^k-x_i\|^2
\leq
4\Delta_{\max}^n(\bx).
\]
Recalling the definition of $v_i^n$, we also have
\begin{align}
v_i^n-v_i^0
&=
\sqrt{
\|\uu_i^0\|^2+
\sum_{k=1}^n
\|\uu_i^k-\uu_i^{k-1}\|^2
}
-
\|\uu_i^0\|
\leq
\sqrt{
\sum_{k=1}^n
\|\uu_i^k-\uu_i^{k-1}\|^2
}.
\label{eq:prop1-stoch-v}
\end{align}
Moreover,
$v_i^{n-1}-v_i^0\leq v_i^n-v_i^0$. Hence, combining
\eqref{eq:prop1-stoch-start},
\eqref{eq:prop1-stoch-metric},
\eqref{eq:prop1-stoch-rho}, and
\eqref{eq:prop1-stoch-v}, we obtain
\begin{align*}
2\sum_{k=1}^n
\eta
\mathbb E[F(\bx^k)-F(\bx)]
&\leq
\mathbb E\!\left[w^0_{\max}\right]
\|\bx^0-\bx\|^2
+
5\sum_{i=1}^{\m}
\mathbb E\!\left[
\Delta_{\max}^n(\bx)
\sqrt{
\sum_{k=1}^n
\|\uu_i^k-\uu_i^{k-1}\|^2
}
\right]
\\
&\quad+
\sum_{k=1}^{n}
\eta^2
\mathbb E
\|\buu^k-\buu^{k-1}\|_{\W_k^{-1}}^2.
\end{align*}
Dividing by $2$ gives~\eqref{Proposition1_bound}. 
\end{proof}

\begin{remark}\label{rk:bounded_domain}
The term $\Delta_{\max}^{\iter}(\bx)$ in \eqref{Proposition1_bound} enters because the block-wise variable metric terms do not telescope. This issue is already present in the deterministic analysis and is shared by block-wise AdaGrad. 
Such quantity is automatically controlled when $\mathrm{dom}\,\reg$ is bounded. In the deterministic Lipschitz-smooth case, Section~\ref{sec:det-lsmooth} shows that this a priori assumption can be removed.
\end{remark}
To provide the convergence rate, we make use of the following result \citep{duchi2011adaptive}.

\begin{lemma}\label{lemma2}
Let $(a_\iter)_{\iter\in \N}\in \R_+^\N$, $n\in \N$ and $\varepsilon > 0$. Then,
\begin{equation*}
    \sum_{\ite=0}^{\iter} \frac{a_\ite}{\varepsilon +\sqrt{ \sum_{j=0}^{\ite} a_j}} \le 2\sqrt{\sum_{\ite=0}^{\iter} a_\ite}.
\end{equation*}
\end{lemma}

\begin{proposition}\label{main_theorem}
Under the  assumptions and notations of Proposition~\ref{Proposition1}, if $\bar\bx^n=n^{-1}\sum_{k=1}^n\bx^k$, then, 
%for every $\bx\in\bHH$,
\begin{align}
\label{extra}
\mathbb E[F(\bar\bx^n)-F(\bx)]
& \leq
\frac{\mathbb E[w^0_{\max}]}{2\eta n}\|\bx^0-\bx\|^2 \notag \\
& \quad
+\frac{\sqrt{\m}}{n}
\left(
\mathbb E\left[
\left(\frac{5\Delta_{\max}^n(\bx)}{2\eta}+\eta\right)^2
\right]
\right)^{1/2}\!\!
\left(
\mathbb E\sum_{k=1}^n
\|\buu^k-\buu^{k-1}\|^2
\right)^{1/2}.
\end{align}
\end{proposition}

\begin{proof}
It follows from
Proposition~\ref{Proposition1} and Lemma~\ref{lemma2} that
\begin{align}
\sum_{k=1}^n
\mathbb E[F(\bx^k)-F(\bx)]
&\leq
\frac{\mathbb E[w^0_{\max}]}
{2\eta}
\|\bx^0-\bx\|^2+
\frac{5}{2\eta}
\sum_{i=1}^{\m}
\mathbb E\left[
\Delta_{\max}^n(\bx)
\sqrt{
\sum_{k=1}^n
\|\uu_i^k-\uu_i^{k-1}\|^2
}
\right]
\notag\\
&\qquad+
\frac{\eta}{2}
\sum_{k=1}^n
\mathbb E
\|\buu^k-\buu^{k-1}\|_{\W_k^{-1}}^2.
\label{eq:main-stoch-start}
\end{align}
Moreover,
\begin{align*}
\sum_{k=1}^n
\|\buu^k-\buu^{k-1}\|_{\W_k^{-1}}^2
&=
\sum_{k=1}^n\sum_{i=1}^{\m}
\frac{1}{\varepsilon+v_i^k}
\|\uu_i^k-\uu_i^{k-1}\|^2
\\
&\leq
\sum_{i=1}^{\m}\sum_{k=1}^n
\frac{
\|\uu_i^k-\uu_i^{k-1}\|^2
}{
\varepsilon+
\sqrt{
\sum_{j=1}^k
\|\uu_i^j-\uu_i^{j-1}\|^2
}
}
\\
&\leq
2\sum_{i=1}^{\m}
\sqrt{
\sum_{k=1}^n
\|\uu_i^k-\uu_i^{k-1}\|^2
},
\end{align*}
where in the last inequality we used Lemma~\ref{lemma2} with
$a_k=\|\uu_i^k-\uu_i^{k-1}\|^2$ for $k>0$ and $a_0=0$.
Hence
\begin{align}
\label{eq:20261028b}
\sum_{k=1}^n
\mathbb E[F(\bx^k)-F(\bx)]
&\leq
\frac{\mathbb E[w^0_{\max}]}
{2\eta}
\|\bx^0-\bx\|^2\notag\\
&\qquad+
\sum_{i=1}^{\m}
\mathbb E\left[
\left(\frac{5\Delta_{\max}^n(\bx)}{2\eta} + \eta\right)
\sqrt{
\sum_{k=1}^n
\|\uu_i^k-\uu_i^{k-1}\|^2
}
\right].
\end{align}
Defining
$\bar\bx^n=n^{-1}\sum_{k=1}^n\bx^k$, Jensen's inequality yields
\[
\mathbb E[F(\bar\bx^n)-F(\bx)]
\leq
\frac1n\sum_{k=1}^n
\mathbb E[F(\bx^k)-F(\bx)].
\]
Moreover, using the inequality
$\norm{\,\cdot\,}_1 \leq \sqrt{d}\norm{\,\cdot\,}_2$
on $\R^d$, 
and using Cauchy-Schwarz's inequality the last term in \eqref{eq:20261028b}, can be bounded as follows
\begin{align*}
\EE\Bigg[\sum_{i=1}^{\m}
&\left(\frac{5\Delta_{\max}^n(\bx)}{2\eta} + \eta\right)
\sqrt{
\sum_{k=1}^n
\|\uu_i^k-\uu_i^{k-1}\|^2
}\Bigg]\\
&\leq
\sqrt{\m}\,
\mathbb E
\left[
\left(\frac{5\Delta_{\max}^n(\bx)}{2\eta} + \eta\right)\left(\sum_{k=1}^n
\|\buu^k-\buu^{k-1}\|^2\right)^{1/2}
\right] \\
& \leq
\sqrt{\m}
\left(
\mathbb E\left[
\left(\frac{5\Delta_{\max}^n(\bx)}{2\eta}+\eta\right)^2
\right]
\right)^{1/2}
\left(
\sum_{k=1}^n
\EE[\|\buu^k-\buu^{k-1}\|^2]
\right)^{1/2}.
\end{align*}
Therefore, the statement follows. 
\end{proof}

\begin{remark}
The bound in Proposition~\ref{main_theorem} shows that the convergence rate is reduced to controlling the accumulation of gradient-differences. The next subsections address this issue in two major settings.
\end{remark}

\subsection{Convergence rate in the Lipschitz-continuous setting}\label{sec:analysis_a}

\begin{proof}[Proof of Theorem~\ref{thm_nonsmooth}]
Considering Proposition~\ref{main_theorem}, we need to control just the sum in bound \eqref{extra}. 
This is possible since under Assumption~\ref{ass:1}, $\|g(\bx)\|\leq G$ for every $\bx$. Hence Assumption~\ref{ass:unbiased} with $\C =0$ gives
\[
\mathbb E \left[\|\buu^k\|^2\right]\leq G^2+\sigma^2,
\]
and therefore
\begin{equation}\label{eq:nonsmooth-stochastic-variation}
\sum_{k=1}^n\EE\left[\|\buu^k-\buu^{k-1}\|^2\right]
\leq 4n(G^2+\sigma^2).
\end{equation}
Combining~\eqref{eq:nonsmooth-stochastic-variation} with Proposition~\ref{main_theorem}, and choosing $\bx\in\argmin F$, the statement follows. 
\end{proof}

\subsection{Convergence rate in the H\"older-smooth setting}\label{sec:analysis_b}
We control the accumulation of gradient-differences directly from the second inequality in Lemma~\ref{lemma1} and the variance bound in Assumption~\ref{ass:unbiased}. We introduce a few preliminary lemmas which will be crucial in the subsequent arguments.
The proofs are given in the Appendix.
\begin{lemma}\label{lemma_recursion2}
Let \(t,a,b,c,\theta,\rho\ge 0\), with \(a>0\),
\(0\leq\rho<\theta\), and \(\theta-\rho\geq 1\).
If $a t^\theta \le b + c t^\rho$, then
\begin{equation*}
t \leq
\left(\frac{b}{a}\right)^{\frac{1}{\theta}}
+
\left(\frac{c}{a}\right)^{\frac{1}{\theta-\rho}}.
\end{equation*}
\end{lemma}

\begin{lemma}\label{lemma:stochastic-gradient-differences}
For every $n\in\N$, with $n\geq1$,
\begin{align*}
\mathbb E \left[\|\buu^n-\buu^{n-1}\|^2\right]
& \leq
\mathbb E \left[\|\buu^n-\nabla f(\bx^n)\|^2\right]
+
2\mathbb E \left[\|\buu^{n-1}-\nabla f(\bx^{n-1})\|^2\right] \\
&\qquad+
2\mathbb E \left[\|\nabla f(\bx^n)-\nabla f(\bx^{n-1})\|^2\right].
\end{align*}
\end{lemma}
We are ready for the main results of the subsection.
\begin{proposition}\label{proposition5}

Let Assumptions~~\ref{ass:0} and \ref{ass:2} hold with
\(\nu\in\left]0,1\right]\), and let Assumptions~\ref{ass:unbiased}
hold with
\(q:=2\nu/(1+\nu)\) and $\bx^\star\in \argmin F$.
Let $(R_n(\eta))_{n \in \N}$ be as in Theorem~\ref{thm_nonsmooth}.
Then, for every $n\in \N$, with $n\geq 1$,
\begin{equation}
\left(
\sum_{k=1}^n \EE\big[\|\buu^k-\buu^{k-1}\|^2\big]
\right)^{1/2}
\hspace{-2ex}\leq
C_{0,\nu}
+C_{1,\nu}n^{\frac{1-\nu}{2(1+\nu)}}
+C_{2,\nu}R_n(\eta)^\nu n^{\frac{1-\nu}{2}}
+\sigma\sqrt{3 n},
\label{eq:stochastic-variation-explicit}
\end{equation}
where
\begin{equation*}
K_\nu:=
\left(\frac{1+\nu}{\nu}\right)^{\frac{2\nu}{1+\nu}}
L^{\frac{2}{1+\nu}}, \qquad H_\nu:=
\begin{cases}
\left[
(3\C)^{\frac{1+\nu}{1-\nu}}
+(2K_\nu)^{\frac{1+\nu}{1-\nu}}
\right]^{\frac{1-\nu}{1+\nu}},
& \nu\in]0,1[,\\[1.5ex]
\max\{3\C,4L\},
& \nu=1.
\end{cases}
\end{equation*}
and
\begin{align}
C_{0,\nu}
&:=\sqrt{2\C}\,[D_f(\bx^0,\bx^\star)]^{\frac{\nu}{1+\nu}},
\qquad
C_{2,\nu}:=\m^{\nu/2}H_\nu^{\frac{1+\nu}{2}},
\notag\\
C_{1,\nu}
&:=\sqrt{H_\nu}
\left(
\frac{1}{2\eta}
\mathbb E \left[\|\bx^0-\bx^\star\|_{\W_0}^2\right]
\right)^{\frac{\nu}{1+\nu}}.
\label{eq:holder-constants}
\end{align}
In particular, if
$\sup_{n\in\N}\mathbb E[(\Delta_{\max}^n(\bx^\star))^2]<+\infty$,
the left-hand side of~\eqref{eq:stochastic-variation-explicit} is
$\mathcal O(n^{(1-\nu)/2}+\sigma\sqrt n)$.
\end{proposition}

\begin{proof}
Set
\begin{equation*}
S_n:=
\sum_{k=1}^n \mathbb E\left[\|\buu^k-\buu^{k-1}\|^2
\right].
\end{equation*}
The optimality condition at $\bx^\star$ implies
$D_f(\by,\bx^\star)\leq F(\by)-F_\star$ for every $\by\in\bHH$.
Thus, applying the second inequality of
Lemma~\ref{lemma1} with $\bx=\bx^\star$, and dividing by $\eta$, gives
\begin{align}
\mathbb E\left[
D_f(\bx^{k+1},\bx^\star)
\right]
& \leq
\frac{1}{2\eta}
\mathbb E\left[
\|\bx^k-\bx^\star\|_{\W_k}^2
-\|\bx^{k+1}-\bx^\star\|_{\W_k}^2
+\rho_k
\right] \notag\\
& \qquad - \frac{q}{2L^{1/\nu}}\mathbb E \left[
\|\nabla f(\bx^{k+1})-\nabla f(\bx^k)\|^{\frac{2}{q}} \right]. \label{eq:prop4.5_main}
\end{align}
For brevity, denote the cumulative Bregman gap 
and the cumulative gradient differences by
\begin{equation*}
A_n:=\sum_{k=1}^n \EE\big[ D_f(\bx^k,\bx^\star)\big]\ \text{and}\ 
B_n:=
\frac{q}{2L^{1/\nu}}
\sum_{k=1}^n
\mathbb E \left[\|\nabla f(\bx^k)-\nabla f(\bx^{k-1})\|^{\frac{2}{q}}\right],
\end{equation*}
respectively.
Summing from $0$ to $n-1$ in \eqref{eq:prop4.5_main}, following the same steps we used to prove Propositions~\ref{Proposition1}, \ref{main_theorem}, it follows that
\begin{equation}
A_n+B_n
\leq
\frac{1}{2\eta}
\mathbb E \left[\|\bx^0-\bx^\star\|_{\W_0}^2\right]
+\sqrt{\m}R_n(\eta)S_n^{1/2}.
\label{eq:AB-control}
\end{equation}
Let $\boldsymbol e^k:=\buu^k-\nabla f(\bx^k)$.
Applying Lemma~\ref{lemma:stochastic-gradient-differences}
and summing from $1$ to $n$, we obtain
\begin{equation*}
S_n
\leq
2\mathbb E \left[\|\boldsymbol e^0\|^2\right]
+3\sum_{k=1}^{n-1}\mathbb E \left[\|\boldsymbol e^k\|^2\right]
+\mathbb E \left[\|\boldsymbol e^n\|^2\right]
+2\sum_{k=1}^n
\mathbb E
\|\nabla f(\bx^k)-\nabla f(\bx^{k-1})\|^2.
\end{equation*}
By Assumption~\ref{ass:unbiased} and Jensen's
inequality (the function $z\mapsto z^q$ is concave), for every $k\in\N$,
\begin{equation*}
\mathbb E \left[\|\boldsymbol e^k\|^2\right]
\leq
\C\big(\mathbb E \left[D_f(\bx^k,\bx^\star)\right]\big)^q
+\sigma^2,
\end{equation*}
and therefore, we obtain
\begin{align}
S_n
& \leq
2\C[D_f(\bx^0,\bx^\star)]^q
+3\C\sum_{k=1}^n
\big(\mathbb E \left[D_f(\bx^k,\bx^\star)\right]\big)^q \notag \\
& \qquad+2\sum_{k=1}^n
\mathbb E
\left[\|\nabla f(\bx^k)-\nabla f(\bx^{k-1})\|^2\right]
+3n\sigma^2.
\label{eq:state-dependent-Sn}
\end{align}
The two sums on the right-hand side can be
controlled in terms of $A_n$ and $B_n$.
H\"older's inequality in $\R^n$, with conjugate exponents
$1/q$ and $1/(1-q)$, yields
\begin{equation*}
\sum_{k=1}^{n}\bigl(\mathbb E \left[D_f(\bx^k,\bx^\star)\right]\bigr)^q
\leq n^{1-q}A_n^q.
\end{equation*}
On the other hand, applying Jensen's inequality to $t\mapsto t^q$ inside each expectation,
and then H\"older's inequality in $\R^\iter$, with conjugate exponents
$1/q$ and $1/(1-q)$, yields
\begin{align*}
\sum_{k=1}^n
\mathbb E\big[
\|\nabla f(\bx^k)-\nabla f(\bx^{k-1})\|^2\big]
&\leq \sum_{k=1}^n
\big(\mathbb E\big[
\|\nabla f(\bx^k)-\nabla f(\bx^{k-1})\|^{2/q}\big]\big)^q\\
&\leq
K_\nu n^{1-q}B_n^q.
\end{align*}
Substituting these two bounds into
\eqref{eq:state-dependent-Sn} gives
\begin{equation}
S_n
\leq
2\C[D_f(\bx^0,\bx^\star)]^q
+n^{1-q}\bigl(3\C A_n^q+2K_\nu B_n^q\bigr)
+3n\sigma^2.
\label{eq:state-dependent-Sn-AB}
\end{equation}
Now, we note that, by the H\"older inequality in $\R^2$,
with conjugate exponents $1/q$ and $1/(1-q)$  (whenever $0<q<1$),
we have that
\begin{equation}
(\forall\,A,B\in \R_+)\quad
3\C A^q+2K_\nu B^q
\leq
H_\nu(A+B)^q.
\label{eq:sharp-H}
\end{equation}
When $q=1$, we have $K_1=2L$, so the same bound
holds with $H_1=\max\{3\C,4L\}$.

Using the subadditivity of $z\mapsto z^q$, together with the bound in \eqref{eq:AB-control}, we obtain
\begin{equation*}
(A_n+B_n)^q
\leq
\left(
\frac{1}{2\eta}
\mathbb E \left[\|\bx^0-\bx^\star\|_{\W_0}^2\right]
\right)^q
+\m^{q/2}(R_n(\eta))^qS_n^{q/2}.
\end{equation*}
Substitution into the estimate for $S_n$,
together with the definitions of $C_{0,\nu}$
and $C_{1,\nu}$, gives
\begin{equation*}
S_n
\leq
C_{0,\nu}^2
+C_{1,\nu}^2n^{1-q}
+3n\sigma^2
+H_\nu\m^{q/2}(R_n(\eta))^qn^{1-q}S_n^{q/2}.
\end{equation*}
It remains to solve this inequality for $S_n$.
Set $y=S_n^{1/2}$. We have
$2-q=2/(1+\nu)$ and
$1-q=(1-\nu)/(1+\nu)$.
Consequently, the definition
$C_{2,\nu}=\m^{\nu/2}H_\nu^{(1+\nu)/2}$
allows us to rewrite the preceding estimate as
\begin{equation*}
y^2
\leq
C_{0,\nu}^2
+C_{1,\nu}^2n^{1-q}
+3n\sigma^2
+
\left(
C_{2,\nu}(R_n(\eta))^\nu n^{\frac{1-\nu}{2}}
\right)^{2-q}
y^q.
\end{equation*}
Since $0<q\leq1$, we have $2-q\geq1$,
so Lemma~\ref{lemma_recursion2} applies with
$\theta=2$ and $\rho=q$. It follows that
\begin{equation*}
y
\leq
\left[
C_{0,\nu}^2
+C_{1,\nu}^2n^{1-q}
+3n\sigma^2
\right]^{1/2}
+
C_{2,\nu}(R_n(\eta))^\nu n^{\frac{1-\nu}{2}}.
\end{equation*}
Finally, using the subadditivity of the square
root and $1-q=(1-\nu)/(1+\nu)$, we conclude that
\begin{equation*}
S_n^{1/2}
\leq
C_{0,\nu}
+C_{1,\nu}n^{\frac{1-\nu}{2(1+\nu)}}
+C_{2,\nu}(R_n(\eta))^\nu n^{\frac{1-\nu}{2}}
+\sigma\sqrt{3n}.
\end{equation*}
This completes the proof. 
\end{proof}

\begin{remark}
The bound in \eqref{eq:stochastic-variation-explicit} depends on $R_n(\eta)$, which in turn depends on $\EE[(\Delta^n_{\max}(\bx^\star))^2]$. Hence, it makes sense whenever $\sup_{n\in \N}\EE[(\Delta^n_{\max}(\bx^\star))^2]$ is finite. In the following result, we provide an alternative bound which does not depend on any boundedness assumption on the iterates, but it is valid only in the deterministic setting.
\end{remark}

\begin{proposition}\label{proposition5_det}
Under Assumptions~\ref{ass:0} and~\ref{ass:2} with
$\nu\in(0,1]$, and for the deterministic oracle
$\buu^\iter=\nabla f(\bx^\iter)$, we have
\[
\left(\sum_{\ite=1}^{\iter}
\|\buu^{\ite}-\buu^{\ite-1}\|^2\right)^{1/2}
=\mathcal O(\iter^{\frac{1-\nu}{2}}).
\]
\end{proposition}

\begin{proof}
Let
\[
I=\left\{
i\in[\m]\;\middle|\;
\exists\,\iter\in\N:
\frac{\eta}{\varepsilon+v_i^\iter}
\leq\frac{\nu}{1+\nu}L^{-1/\nu}
\right\}.
\]
Then, for every $i\in [\m]\setminus I$ and every $n\in \N$, $\eta/(\varepsilon+ v^\iter_i)>\frac{\nu}{1+\nu}L^{-1/\nu}$ and hence
\begin{equation*}
\sum_{\ite=0}^{\iter} \norm{\uu_i^{\ite+1} - \uu^{\ite}_i}^2 \leq 
\sum_{\ite=1}^{\iter+1} \norm{\uu_i^{\ite} - \uu^{\ite-1}_i}^2 \leq (v_i^{\iter+1})^2 < [\eta(1+\nu)L^{1/\nu}/\nu-\varepsilon]^2.
\end{equation*}
Therefore, for every $n\in\N$
\begin{equation}
\label{eq:20260124b}
\sum_{\ite=0}^\iter \sum_{i\in [\m]\setminus I} \norm{\uu_i^{\ite+1} - \uu^{\ite}_i}^2 =
\sum_{i\in [\m]\setminus I}\sum_{\ite=0}^\iter \norm{\uu_i^{\ite+1} - \uu^{\ite}_i}^2 < \m[\eta(1+\nu)L^{1/\nu}/\nu-\varepsilon]^2.
\end{equation}
Thus, it follows from \eqref{eq:20260124b} that if $I= \varnothing$,  
\begin{equation*}
(\forall\,\iter\in \N)\quad
\sum_{\ite=0}^\iter \norm{\buu^{\ite} - \buu^{\ite-1}}^2 = \sum_{\ite=0}^\iter \sum_{i\in [\m]}
\norm{\uu_i^{\ite} - \uu^{\ite-1}_i}^2 \leq \m[\eta(1+\nu)L^{1/\nu}/\nu-\varepsilon]^2
\end{equation*}
and the statement follows.
On the other hand, suppose that $I\neq\varnothing$ and set
\begin{equation*}
(\forall\,i\in I)\quad \iter_i = \min\bigg\{ \iter\in \N\,\Big\vert\, \frac{\eta}{\varepsilon+ v^\iter_i} \leq 
\frac{\nu}{1+\nu}L^{-1/\nu}\bigg\}\quad\text{and}\quad \bar{\iter} = \max_{i\in I} \iter_i.
\end{equation*}
Moreover, for every $i\in I$ and for any $k \ge \bar n$, the following holds.
\begin{align}
\label{eq:holder_coordinate_bound}
\frac{\eta}{\varepsilon+v_i^\ite}
\|\uu_i^{\ite+1}-\uu_i^\ite\|^2
&\leq
\frac{\nu}{1+\nu}L^{-1/\nu}
\|\uu_i^{\ite+1}-\uu_i^\ite\|^{\frac2q}
+
\eta(2+\varepsilon^{-1})(v_i^{\ite+1}-v_i^\ite).
\end{align}
Indeed, if $\|\uu_i^{\ite+1}-\uu_i^\ite\|\geq1$,
the first term suffices because $q\le 1$.
Otherwise, the second term suffices, since
\[
\|\uu_i^{\ite+1}-\uu_i^\ite\|^2
=(v_i^{\ite+1}+v_i^\ite)(v_i^{\ite+1}-v_i^\ite),
\qquad
v_i^{\ite+1}+v_i^\ite\leq2v_i^\ite+1.
\]
In the deterministic setting, we have $\buu^n=\nabla f(\bx^n)$ and the basic inequality in \eqref{lemma1_bound2}, with $\iter=\ite$ and $\bx= \bx^{\ite}$, reduces to
\begin{align}
\label{eq:holder_descent}
2\bigl(F(\bx^{\ite+1})-F(\bx^\ite)\bigr)
&\leq
\sum_{i=1}^{\m}
\frac{\eta}{\varepsilon+v_i^\ite}
\|\uu_i^{\ite+1}-\uu_i^\ite\|^2  -
\frac{2\nu}{1+\nu}L^{-1/\nu}
\|\buu^{\ite+1}-\buu^\ite\|^{\frac2q},
\end{align}
where we used $1 + 1/\nu = 2/q$. Splitting the coordinates into $I$ and its complement, and using
\eqref{eq:holder_coordinate_bound}, we obtain, for every
$\ite\geq\bar{\iter}$,
\begin{align*}
2\bigl(F(\bx^{\ite+1})-F(\bx^\ite)\bigr)
&\leq
\sum_{i\in I}
\frac{\eta}{\varepsilon+v_i^\ite}
\|\uu_i^{\ite+1}-\uu_i^\ite\|^2
+
\sum_{i\notin I}
\frac{\eta}{\varepsilon+v_i^\ite}
\|\uu_i^{\ite+1}-\uu_i^\ite\|^2\\
&\qquad-
\frac{2\nu}{1+\nu}L^{-1/\nu}
\|\buu^{\ite+1}-\buu^\ite\|^{\frac2q}\\
&\leq
\eta(2+\varepsilon^{-1})
\sum_{i\in I}(v_i^{\ite+1}-v_i^\ite)
+
\frac{\eta}{\varepsilon}
\sum_{i\notin I}\|\uu_i^{\ite+1}-\uu_i^\ite\|^2\\
&\qquad+
\frac{\nu}{1+\nu}L^{-1/\nu}
\sum_{i\in I}
\|\uu_i^{\ite+1}-\uu_i^\ite\|^{\frac2q} \\
&\qquad-
\frac{2\nu}{1+\nu}L^{-1/\nu}
\|\buu^{\ite+1}-\buu^\ite\|^{\frac2q}\\
&\leq
\eta(2+\varepsilon^{-1})
\sum_{i\in I}(v_i^{\ite+1}-v_i^\ite)
+
\frac{\eta}{\varepsilon}
\sum_{i\notin I}\|\uu_i^{\ite+1}-\uu_i^\ite\|^2\\
&\qquad-
\frac{\nu}{1+\nu}L^{-1/\nu}
\|\buu^{\ite+1}-\buu^\ite\|^{\frac{2}{q}},
\end{align*}
where the last inequality follows from $\ell_p$-norm monotonicity.
For any $n > \bar n$, rearranging, summing from $\bar{\iter}$ to $\iter-1$ and using \eqref{eq:20260124b}, gives
\begin{align*}
\frac{\nu}{1+\nu}L^{-\frac1\nu}
\sum_{\ite=\bar{\iter}}^{\iter-1}
\|\buu^{\ite+1}-\buu^\ite\|^{\frac2q}
&\leq
2\bigl(F(\bx^{\bar{\iter}})-F(\bx^\iter)\bigr)
+
\frac{\eta}{\varepsilon}
\sum_{\ite=\bar{\iter}}^{\iter-1}
\sum_{i\notin I}\|\uu_i^{\ite+1}-\uu_i^\ite\|^2\\
&\quad+
\eta(2+\varepsilon^{-1})
\sum_{i\in I}(v_i^\iter-v_i^{\bar{\iter}})\\
& \le 2\bigl(F(\bx^{\bar{\iter}})-F_\star\bigr)
+
\frac{\eta}{\varepsilon}
\sum_{\ite=\bar{\iter}}^{\iter-1}
\sum_{i\notin I}\|\uu_i^{\ite+1}-\uu_i^\ite\|^2\\
&\quad+
\eta(2+\varepsilon^{-1})
\sum_{i=1}^\m\left(\sum_{k=\bar n}^{n-1}\lVert \uu_i^{k+1} - \uu_i^k \rVert^2\right)^{1/2}\\[2ex]
&\leq
2\bigl(F(\bx^{\bar{\iter}})-F_\star\bigr)
+
\frac{\eta}{\varepsilon}\m[\eta(1+\nu)L^{1/\nu}/\nu-\varepsilon]^2\\
&\quad+
\eta(2+\varepsilon^{-1})
\sqrt \m \left(
\sum_{\ite=\bar{\iter}}^{\iter-1}
\|\buu^{\ite+1}-\buu^\ite\|^2
\right)^{1/2},
\end{align*}
where in the last step we used the inequality $\norm{\cdot}_1 \le \sqrt\m \norm{\cdot}_2$ on $\R^\m$. Applying Jensen's inequality in $\R^{n-\bar n}$ on the left-hand side of the latter inequality, we obtain 
\begin{align*}
\frac{\nu}{1+\nu}L^{-1/\nu}
(\iter-\bar{\iter})^{1- \frac{1}{q}}
\left(
\sum_{\ite=\bar{\iter}+1}^{\iter}
\|\buu^{\ite}-\buu^{\ite-1}\|^2
\right)^{\frac{1}{q}} &\leq 2\bigl(F(\bx^{\bar{\iter}})-F_\star\bigr)
+
\frac{\eta}{\varepsilon}\m[\eta(1+\nu)L^{1/\nu}/\nu-\varepsilon]^2 \\
& \qquad+ 
\eta(2+\varepsilon^{-1})\sqrt{\m}
\left(
\sum_{\ite=\bar{\iter}+1}^{\iter}
\|\buu^{\ite}-\buu^{\ite-1}\|^2
\right)^{1/2}.
\end{align*}
Set
$y:=\sum_{\ite=\bar{\iter}+1}^{\iter}
\|\buu^{\ite}-\buu^{\ite-1}\|^2$. The inequality above has the form
\[
A(\iter-\bar{\iter})^{1-1/q}y^{1/q}
\leq B+Cy^{1/2},
\]
for some constants $A>0$ and $B,C\geq0$ independent of $\iter$.
Applying Lemma~\ref{lemma_recursion2} with
$t=y^{1/2}$, $\theta=2/q$, $\rho=1$, we obtain
\[
y^{1/2}
\leq
\left(\frac{B}{A}\right)^{q/2}
(\iter-\bar{\iter})^{(1-q)/2}
+
\left(\frac{C}{A}\right)^{q/(2-q)}
(\iter-\bar{\iter})^{(1-q)/(2-q)}.
\]
Recalling that
$q=2\nu/(1+\nu)$, we have
$\left(\sum_{\ite=\bar{\iter}+1}^{\iter}
\|\buu^{\ite}-\buu^{\ite-1}\|^2\right)^{1/2}=\mathcal O(\iter^{(1-\nu)/2})$. Adding the finitely many terms
preceding $\bar{\iter}$ proves the statement.
\end{proof}

\begin{remark} 
Proposition~\ref{proposition5_det} provides direct control of the adaptive accumulation in the deterministic H\"older-smooth setting, without requiring boundedness of either the domain or the iterates. The key point was to apply the basic inequality with $\bx=\bx^\iter$, which eliminates the non-telescopic contributions, and to use the negative term on the right-hand side of \eqref{lemma1_bound2} to control the gradient differences.
Finally, note that in the Lipschitz-smooth case ($\nu=1$), Proposition~\ref{proposition5_det} yields 
summability of the squared gradient-differences. This aspect will be crucial in the next section where convergence of the iterates 
is analyzed.
\end{remark}

\begin{proof}[Proof of Theorem~\ref{prop:accumulator_control}]
The statement follows from
Proposition~\ref{proposition5} and Proposition~\ref{proposition5_det}. 
\end{proof}

\begin{proof}[Proof of Theorem~\ref{thm_smooth}]
Combining Proposition~\ref{proposition5} with Proposition~\ref{main_theorem} yields the statement. 
\end{proof}

\begin{remark}\label{rk:comparison}
The work \citep{rodomanov2024universality} complements a bounded-variance analysis
with a refined assumption controlling the variation of the oracle noise between
arbitrary points. For unbiased oracles, writing
$e(\bx,\zeta):=\hat g(\bx,\zeta)-\nabla f(\bx)$, their condition requires the existence of some $\bar L,\bar\delta\geq0$ such that
\begin{equation}
\label{ass:rodomanov}
    \forall\bx,\by\in\bHH\colon\qquad
    \mathbb E \left[\|e(\bx,\zeta)-e(\by,\zeta)\|^2\right]
    \leq
    2\bar L\bigl(D_f(\bx,\by)+\bar\delta\bigr).
\end{equation}
A relevant instance considered therein is
$f=\mathbb E_\zeta[f_\zeta]$, with every $f_\zeta$ convex and H\"older-smooth
with exponent $\nu$ and constant $L_{\zeta}$. If
$L:=\sup_\zeta L_{\zeta} < +\infty$, then $f$ is H\"older-smooth with the
same exponent and constant $L$. For
$\nu\in]0,1[$, Lemma~16 and Remark~35 of
\citep{rodomanov2024universality} show that
\eqref{ass:rodomanov} holds for every $\bar\delta>0$ with
\[
    \bar L
    =
    \left(
        \frac{1-\nu}{2(1+\nu)\bar\delta}
    \right)^{\frac{1-\nu}{1+\nu}}
    L^{\frac{2}{1+\nu}},
\]
while for $\nu=1$, it is sufficient to set $\bar\delta = 0$ and $\bar L = L$.
Setting $q:=2\nu/(1+\nu)$ and optimizing over $\bar\delta$ yields
\[
    \mathbb E \left[\|e(\bx,\zeta)-e(\by,\zeta)\|^2\right]
    \leq
    \left(\frac{1+\nu}{\nu}\right)^q
    L^{\frac{2}{1+\nu}}
    [D_f(\bx,\by)]^q.
\]
Hence, taking $\by=\bx^\star$, their assumptions imply
Assumption~\ref{ass:unbiased} with
\begin{equation}
\label{eq:rodomanov-to-ours}
    \C=
    2\left(\frac{1+\nu}{\nu}\right)^q
    L^{\frac{2}{1+\nu}},
    \qquad
    \sigma^2=2\sigma_\star^2,
\end{equation}
where $\sigma_\star^2:=\mathbb E \left[\|e(\bx^\star,\zeta)\|^2\right]$. The resulting rates can therefore be compared explicitly. We defer the reader to the Appendix for a detailed comparison of the convergence rates and further considerations on the variance condition.
\end{remark}

\subsection{Convergence of the iterates in the deterministic Lipschitz-smooth setting}\label{sec:det-lsmooth}
We finally return to the deterministic Lipschitz-smooth setting ($\nu=1$), with $\buu^n=\nabla f(\bx^n)$. In this setting,
the convergence of the iterates for proximal gradient methods classically follows from the so-called Fej\'er monotonicity property.
In this respect, now $\sigma=0$ and
inequality \eqref{lemma1_bound} in Lemma~\ref{lemma1} with $\bx\in \argmin F$ yields 
\begin{equation*}
    \norm{\bx^{\iter+1} - \bx}_{\W_\iter}^2\leq 
     \norm{\bx^\iter - \bx}_{\W_\iter}^2 + \eta^2 \norm{\buu^{\iter+1} - \buu^{\iter}}_{\W_\iter^{-1}}^2. 
\end{equation*}
This already resembles a Fej\'er-type inequality in a variable metric,
except for a slight mismatch in the metric indices of the distance terms. We briefly recall the main definition and properties of quasi-Fej\'er sequences.

\begin{definition}[Definition~3.1(ii) in \citep{combettes2013variable}]
\label{def:fejer}
Let $(V_\iter)_{\iter\in \N}$ be a sequence of positive self-adjoint bounded linear operators on the Hilbert space $\HH$,
such that, for every $\iter\in \N$, $V_n \succeq \alpha \Id$, for some $\alpha>0$.
Let $\norm{\cdot}_{V_n}$ be the Hilbert norms defined by the $V_\iter$'s
and let $S \subset \HH$ be a nonempty set.
A sequence ${(z^\iter)}_{\iter \in \N}$ in $\HH$ is a 
\emph{quasi-Fej\'er  sequence with respect to $S$
relative to ${(V_\iter)}_{\iter \in \N}$}  if
there exist summable sequences ${(\alpha_\iter)}_{\iter \in \N}$
and ${(\chi_\iter)}_{\iter \in \N}$ in $\R_{++}$ such that
\begin{equation*}
%\label{eq:fejerprop}
(\forall\, z \in S)(\forall\, \iter \in \N)\qquad \norm{z^{\iter+1} - z}_{V_{\iter+1}}^2 
\leq (1+\chi_\iter) \norm{z^\iter - z}_{V_\iter}^2 + \alpha_\iter.
\end{equation*}
\end{definition}

\begin{fact}[Lemma~2.3, Proposition~3.2, and Theorem~3.3 in \citep{combettes2013variable}]
\label{fact:fejer}
Under the assumptions of Definition~\ref{def:fejer} suppose additionally that $\sup_{\iter\in \N} \norm{V_n}<+\infty$. Consider the two statements
\setlist[enumerate]{itemsep=0mm}
\begin{enumerate}[label={\rm (\alph*)}, leftmargin=4ex]
\item\label{fact:fejer_a} either $(\forall\,\iter\in \N)\ V_{\iter+1} \succeq V_{\iter}$ or $(\forall\,\iter\in \N)\ V_{\iter+1} \preceq V_{\iter}$.
\item\label{fact:fejer_b} There exists a positive self-adjoint bounded linear operator $V\colon \HH\to \HH$ such that, for every $z\in \HH$, $V_\iter z \to Vz$.
\end{enumerate}
Then \ref{fact:fejer_a} $\ \Rightarrow\ $ \ref{fact:fejer_b}. Moreover, suppose that $S\subset\HH$
is a nonempty set and let $(z^\iter)_{\iter\in \N}$ be a quasi-Fej\'er sequence in $\HH$ with respect to $S$
relative to $(V_\iter)_{\iter\in \N}$. Then if \ref{fact:fejer_b} holds, we have
\begin{enumerate}[label={\rm (\roman*)}, leftmargin=4ex]
    \item $(z^\iter)_{\iter\in \N}$ is bounded and, for every $z\in S$, $(\norm{z^\iter-z}_{V_\iter})_{\iter\in \N}$ is convergent
    \item $(z^\iter)_{\iter\in \N}$ is weakly convergent to a point in $S$ iff every weak sequential cluster point of $(z^\iter)_{\iter\in \N}$ belongs to $S$.
\end{enumerate}
\end{fact}
We now return to Algorithm~\ref{adadiff}.
In the Lipschitz-smooth setting,
Theorem~\ref{prop:accumulator_control}\ref{proposition5_ii} yields summability of the squared gradient differences,
from which the following proposition follows.

\begin{proposition}
\label{cor:20260126a}
Under Assumptions~\ref{ass:0} and \ref{ass:2}, with $\nu=1$, and for the deterministic oracle, let $(\bx^\iter)_{\iter\in \N}$ be generated by Algorithm~\ref{adadiff}.
% with $\varepsilon>0$. 
 Then,
for all $\iter\in \N$, $\W_\iter\succeq \varepsilon \Id$ and $\sup_{\iter\in \N}\norm{\W_\iter}<+\infty$.
\end{proposition}
\begin{proof}
Recalling \eqref{AdaDiff} it is clear that, for every $\iter\in \N$ and $i\in\ [\m]$, $\w_i^\iter\geq \varepsilon$, so that
$\W_\iter \succeq \varepsilon \Id$. Moreover,
\begin{equation*}
(\forall\,\iter\in \N)(\forall\,i\in [\m])\quad w^\iter_i = \varepsilon + \sqrt{\sum_{\ite=0}^\iter \norm{\uu_i^\ite- \uu_i^{\ite-1}}^2} \leq \varepsilon + \sqrt{\sum_{\ite=0}^\iter \norm{\buu^\ite- \buu^{\ite-1}}^2},
\end{equation*}
so that $\sup_{\iter\in\N}\norm{\W_\iter}=\sup_{\iter\in\N}(\max_{1\leq i\leq \m} \w_i^\iter) \leq \varepsilon + (\sum_{\iter=0}^{+\infty} \norm{\buu^\iter- \buu^{\iter-1}}^2)^{1/2}<+\infty$. 
\end{proof}

\begin{proposition}
\label{prop:Fejer}
Under Assumptions~\ref{ass:0} and~\ref{ass:2} with
$\nu\in(0,1]$, and for the deterministic oracle
$\buu^\iter=\nabla f(\bx^\iter)$, 
the sequence $(\bx^{\iter+1})_{\iter\in \N}$ generated by Algorithm~\ref{adadiff}
is a quasi-Fej\'er  with respect to $\argmin F$ relative to $(\W_\iter)_{\iter\in \N}$, that is,
for every $\iter > 0$ and $\bx\in \argmin F$
\begin{equation*}
\norm{\bx^{\iter+1}-\bx}^2_{\W_{\iter}} \leq (1+\chi_\iter) \norm{\bx^{\iter}-\bx}^2_{\W_{\iter-1}} + \alpha_\iter
\end{equation*}
where $(\chi_\iter)_{\iter\in \N}$ and $(\alpha_\iter)_{\iter\in \N}$ are summable, defined below
\begin{equation*}
\chi_\iter :=\max_{1\leq i\leq \m}\frac{\varepsilon + v_i^{\iter}}{\varepsilon + v_i^{\iter-1}}- 1, \quad
\alpha_\iter := \eta^2 \norm{\buu^{\iter+1} - \buu^{\iter}}_{\W_\iter^{-1}}^2,
\end{equation*}
with the convention $v_i^{-1}=0$. 
\end{proposition}

\begin{proof}
Let $\bx\in \argmin F$. Then it follows from inequality \eqref{lemma1_bound} in Lemma~\ref{lemma1} that
\begin{equation*}
(\forall\,\iter\in \N)\quad
\norm{\bx^{\iter+1} - \bx}_{\W_\iter}^2\leq \norm{\bx^\iter - \bx}_{\W_\iter}^2+ \eta_\iter^2 \norm{\buu^{\iter+1} - \buu^{\iter}}_{\W_\iter^{-1}}^2.
\end{equation*}
On the other hand
\begin{align*}
\norm{\bx^\iter - \bx}_{\W_\iter}^2 &= \sum_{i=1}^\m (\varepsilon + v_i^\iter) \norm{\x_i^\iter- \x_i}^2 \\
& \leq \max_{1\leq i\leq \m}\frac{\varepsilon + v_i^{\iter}}{\varepsilon + v_i^{\iter-1}}\sum_{i=1}^\m (\varepsilon + v_i^{\iter-1}) \norm{\x_i^{\iter}- \x_i}^2\\[1ex]
& = \max_{1\leq i\leq \m}\frac{\varepsilon + v_i^{\iter}}{\varepsilon + v_i^{\iter-1}}\norm{\bx^\iter - \bx}_{\W_{\iter-1}}^2 = (1+\chi_\iter)\norm{\bx^\iter - \bx}_{\W_{\iter-1}}^2,
\end{align*}
where, for every $\iter\in \N$, 
\begin{align*}
0 \leq \chi_\iter & := \max_{1\leq i\leq \m}\frac{\varepsilon + v_i^{\iter}}{\varepsilon + v_i^{\iter-1}}- 1
= \max_{1\leq i\leq \m} \frac{v_i^{\iter} - v_i^{\iter-1}}{\varepsilon + v_i^{\iter-1}}  \\ 
&\leq \max_{1\leq i\leq \m} \frac{v_i^{\iter} - v_i^{\iter-1}}{\varepsilon} \leq \frac{1}{\varepsilon}\sum_{i=1}^\m (v_i^{\iter} - v_i^{\iter-1}),
\end{align*}
so that, 
\begin{align*}
\sum_{\ite=0}^{\iter} \chi_\ite &\leq \frac{1}{\varepsilon}\sum_{\ite=0}^{\iter} \sum_{i=1}^\m(v_i^{\ite} - v_i^{\ite-1}) = \frac{1}{\varepsilon}\sum_{i=1}^\m \sum_{\ite=0}^{\iter} (v_i^{\ite} - v_i^{\ite-1}) \\
& \leq \frac{1}{\varepsilon} \sum_{i=1}^\m v_i^\iter = \frac{1}{\varepsilon}\sum_{i=1}^{\m} \sqrt{\sum_{\ite=0}^{\iter} \norm{\uu_i^{\ite}-\uu_i^{\ite-1}}^2}
\end{align*}
and hence
$\sum_{\iter=0}^{+\infty} \chi_\iter \leq \varepsilon^{-1}\m (\sum_{\iter=0}^{+\infty} \norm{\buu^{\iter}-\buu^{\iter-1}}^2)^{1/2}<+\infty$, by Proposition~\ref{proposition5_det}. Moreover,
\begin{equation*}
\alpha_\iter:=\eta_\iter^2 \norm{\buu^{\iter+1} - \buu^{\iter}}_{\W_\iter^{-1}}^2 = \eta^2 \norm{\buu^{\iter+1} - \buu^{\iter}}_{\W_\iter^{-1}}^2 \leq \frac{\eta^2}{\varepsilon} \norm{\buu^{\iter+1} - \buu^{\iter}}^2
\end{equation*}
and the statement follows again from Proposition~\ref{proposition5_det}. 
\end{proof}
Fact~\ref{fact:fejer} implies that $(\bx^n)$ is bounded. The same result then guarantees weak convergence of the entire sequence provided that every weak sequential cluster point belongs to $\argmin F$. We establish this remaining property in the following proposition.

\begin{proposition}\label{prop:clusterpoints}
Every weak sequential cluster point of $(\bx^{\iter})_{\iter\in \N}$ belongs to $\argmin F$.
\end{proposition}

\begin{proof}
First, we show that there exists $(\bm{u}^{\iter})_{\iter\in \N}$ such that, for every $n\in \N$, $\bm{u}^\iter\in \partial F(\bx^\iter)$ and
 $\bm{u}^\iter\to 0$.
Let $\iter\in \N$.
As we saw in the proof of Lemma~\ref{lem:basicPGstep}, we have
\begin{equation*}
\frac{1}{\eta}\W_\iter(\bx^{\iter}-\bx^{\iter+1})- \buu^\iter\in \partial \reg(\bx^{\iter+1})
\end{equation*}
Therefore,
\begin{equation*}
\bm{u}^{\iter+1}:= \frac{1}{\eta}\W_\iter(\bx^{\iter}-\bx^{\iter+1}) + \buu^{\iter+1} - \buu^\iter \in \partial f(\bx^{\iter+1})+\partial \varphi(\bx^{\iter+1}) = \partial F(\bx^{\iter+1})
\end{equation*}
Now we observe that, setting $\beta= \sup_{\iter\in \N}\norm{\W_\iter}$ (recall Corollary~\ref{cor:20260126a}), we have
\begin{equation}
\label{eq:20260126b}
\norm{\bm{u}^{\iter+1}}^2 \leq \frac{2\beta^2}{\eta^2}\norm{\bx^{\iter+1}-\bx^\iter}^2 + 2\norm{\buu^{\iter+1} - \buu^\iter}^2.
\end{equation}
Moreover, inequality \eqref{lemma1_bound} in Lemma~\ref{lemma1} with $\bx= \bx^\iter$ yields
\begin{equation*}
 \norm{\bx^{\iter+1} - \bx^\iter}_{\W_\iter}^2\leq 
2\eta(F(\bx^\iter)- F(\bx^{\iter+1}))  
+ \eta^2 \norm{\buu^{\iter+1} - \buu^{\iter}}_{\W_\iter^{-1}}^2
\end{equation*}
and since $\W_\iter \succeq \varepsilon \Id$ (by Corollary~\ref{cor:20260126a}) we have
\begin{equation*}
 \varepsilon \norm{\bx^{\iter+1} - \bx^\iter}^2\leq 
2\eta(F(\bx^\iter)- F(\bx^{\iter+1}))  
+ \frac{\eta^2}{\varepsilon} \norm{\buu^{\iter+1} - \buu^{\iter}}^2.
\end{equation*}
It follows from the above inequality and \eqref{eq:20260126b} that
\begin{equation*}
    \norm{\bm{u}^{\iter+1}}^2 \leq \frac{4\beta^2}{\eta\varepsilon} (F(\bx^\iter)- F(\bx^{\iter+1}))  +
    2\bigg( \frac{\beta^2}{\varepsilon^2} + 1\bigg)
    \norm{\buu^{\iter+1} - \buu^{\iter}}^2. 
\end{equation*}
Thus, by Proposition~\ref{proposition5_det},
 \begin{equation*}
 \sum_{\iter=0}^{+\infty} \norm{\bm{u}^{\iter+1}}^2 \leq \frac{4\beta^2}{\eta\varepsilon}(F(\bx^0)- F_\star)
 + 2\bigg( \frac{\beta^2}{\varepsilon^2} + 1\bigg)\sum_{\iter=0}^{+\infty} \norm{\buu^{\iter+1} - \buu^{\iter}}^2<+\infty
 \end{equation*}
 and hence $\bm{u}^{\iter}\to 0$. Now, let $(\bx^{\iter_k})_{k\in \N}$ be a subsequence of
 $(\bx^\iter)_{\iter\in \N}$ such that $\bx^{\iter_k} \rightharpoonup \bar{\bx}$. Then we have
 \begin{equation*}
 (\bx^{\iter_k}, \bm{u}^{\iter_k}) \in \mathrm{gra}(\partial F)\quad\text{and}\quad
 \bx^{\iter_k} \rightharpoonup \bar{\bx},  \bm{u}^{\iter_k} \to 0.
 \end{equation*}
 Since $F\in \Gamma_0(\bHH)$, the graph of $\partial F$ is weakly-strongly closed, meaning closed in the product of the topological spaces $\bHH_w\times \bHH$, where $\bHH_w$ is endowed with the weak topology and $\bHH$ is endowed with the strong topology.  Therefore $(\bar{\bx},0)\in \mathrm{gra}(\partial F)$ and hence $0\in \partial F(\bar{\bx})$. The statement follows. 
\end{proof}

\begin{proof}[Proof of Theorem~\ref{thm_smooth_det}]
Proposition~\ref{prop:Fejer} ensures that the sequence $(\bx^\iter)_{\iter\in \N}$, being quasi-Fej\'er, is bounded, so that
$(\Delta^{\iter}_{\max})_{\iter\in \N}$ is bounded. 
Concerning the weak convergence of the iterates, in view of Fact~\ref{fact:fejer}, it is a direct consequence of Proposition~\ref{prop:Fejer} and Proposition~\ref{prop:clusterpoints}. The convergence rate follows from Theorem~\ref{thm_smooth}, with $\sigma = 0$ and $\C = 0$. 
\end{proof}

\begin{remark}
The improvement above is specific to the deterministic Lipschitz-smooth setting. Under Assumption~\ref{ass:unbiased} with $\sigma>0$, Proposition~\ref{proposition5} contains the term $\sigma^2n$, so the stochastic gradient differences are not guaranteed to be summable. Whether the interpolation case $\sigma=0$ is sufficient to restore summability without any boundedness requirement, and consequently the iterate-convergence mechanism above, is left for future work.
\end{remark}

\section{Numerical experiments}\label{4}

We now investigate the empirical behavior of AdaGrad-Diff. The goal of this section is to illustrate how the difference-based metric affects the behavior of the method. AdaGrad is the natural baseline, as it uses the same adaptive structure but constructs its metric from cumulative gradient magnitudes rather than successive gradient differences.

We begin with the pathological example motivating our analysis, which isolates the effect of the two accumulation rules. We then consider deterministic and stochastic finite-sum problems of the form
\begin{equation}
    \min_{\bx\in\R^\m}
    f(\bx)+\varphi(\bx)
    := \frac{1}{N}\sum_{j=1}^N \ell(\bx,z_j)+\varphi(\bx),
    \label{eq:finite_sum_experiments}
\end{equation}
where $z_j=(a_j,b_j)\in\R^\m\times\R$ represents an input-output pair that we stack together in the matrix $A$ and vector $\textbf{b}$, respectively. For all the algorithms we consider, we set
$\varepsilon=10^{-8}$ for numerical stability.
\subsection{Pathological example}

\begin{wrapfigure}{r}{0.35\linewidth}
    \vspace{-1.8em}
    \centering
    \includegraphics[width=\linewidth]{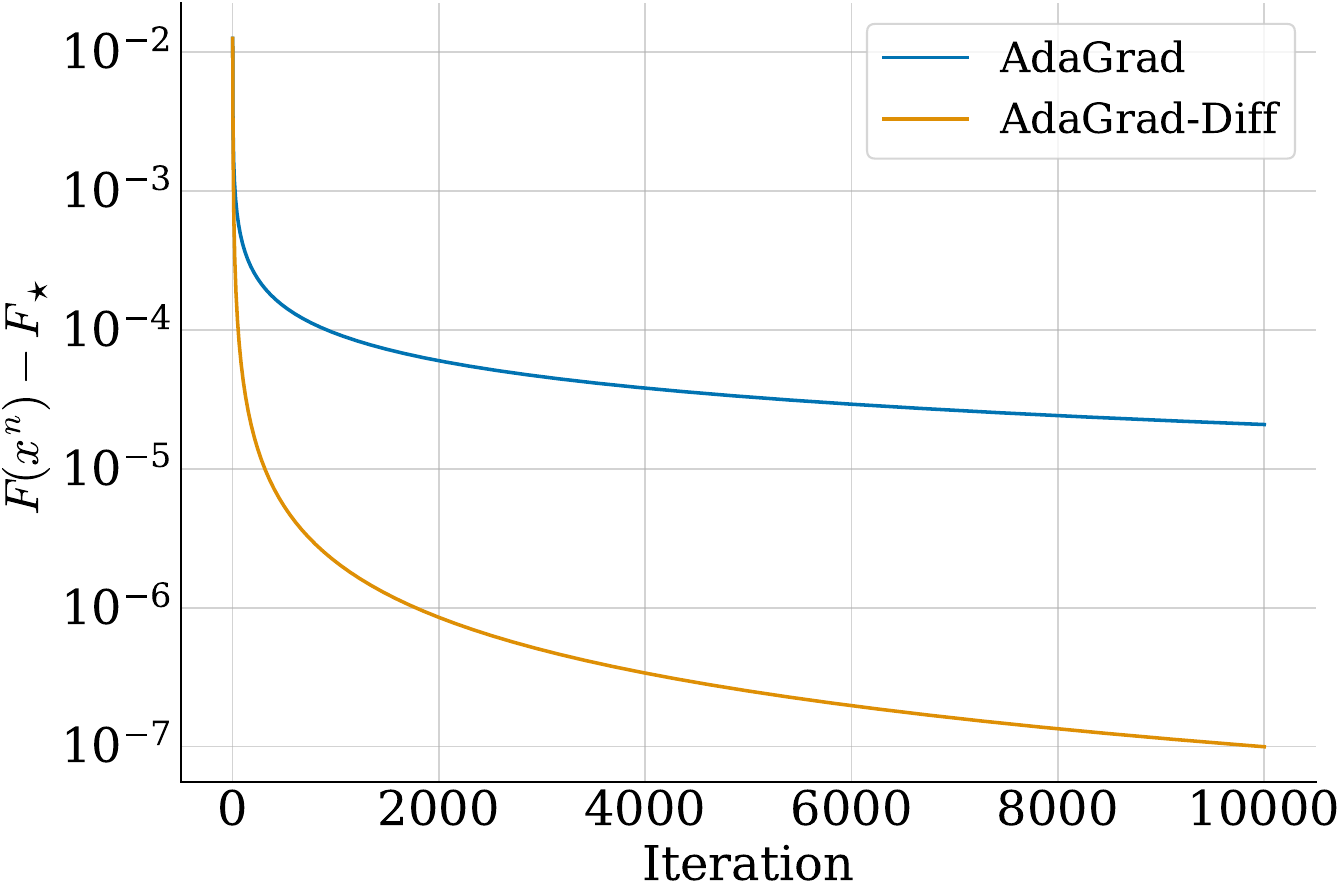}
    \caption{\small Objective gap for \eqref{def_f_reg} with $p=8$, $c=0.1$, $x^0=1$, $\eta=1$.}
    \label{fig:pathological-example}
    \vspace{-0.8em}
\end{wrapfigure}
We begin with the pathological composite problem used in our theoretical analysis of AdaGrad. As shown in Figure~\ref{fig:pathological-example}, AdaGrad-Diff reaches a substantially smaller objective gap, while AdaGrad converges at a markedly slower rate.
This behavior is consistent with the mechanism identified by our analysis. The accumulation of gradient magnitudes causes the stepsizes of AdaGrad to decay excessively, whereas accumulating successive gradient differences allows AdaGrad-Diff to retain the faster rate predicted by the theory.

\subsection{Experiments in the deterministic setting}

We first compare AdaGrad and AdaGrad-Diff on two deterministic composite
problems, and then compare AdaGrad-Diff with UniSgd on a higher-dimensional
smooth problem. The results are averaged over 10 independent
initializations: curves and shaded regions report the average and standard deviation.

\textbf{AdaGrad versus AdaGrad-Diff.}
We consider a problem satisfying Assumption~\ref{ass:1}, i.e. a regression problem with the Least Absolute Deviation (LAD) loss, with $\ell_1$-regularization
\begin{equation*}
    \min_{\bx\in\R^\m}\;
    \frac{1}{N}\sum_{j=1}^N |b_j-\langle a_j,\bx\rangle|
    +\lambda\|\bx\|_1.
\end{equation*}
The matrix $A$ has independent $\mathcal N(0,1)$ entries and
$\boldsymbol b=A\bar\bx+\epsilon$, where $\epsilon\in\R^N$ is
Gaussian noise and $\bar\bx$ has 20 nonzero entries. We set $N=500$,
$\m=100$, $\lambda=10^{-2}$, and use 100 iterations.

We also consider a problem satisfying Assumption~\ref{ass:2}, i.e. the sparse logistic regression problem
\begin{equation}
    \min_{\bx\in\R^\m}\;
    \frac{1}{N}\sum_{j=1}^N
    \log\bigl(1+\exp(-b_j\langle a_j,\bx\rangle)\bigr)
    +\lambda\|\bx\|_1,
    \label{eq:logistic_l1_exp}
\end{equation}
on the \texttt{splice} dataset from LIBSVM\footnote{\url{https://www.csie.ntu.edu.tw/~cjlin/libsvmtools/datasets/}\label{libsvm}},
with $N=800$, $\m=60$, $\lambda=10^{-2}$, and a budget of
500 iterations. Here the logistic loss is smooth, while the $\ell_1$
penalty is handled by the proximal step.

\textbf{UniSgd versus AdaGrad-Diff.}
To compare coordinate-wise and the corresponding scalar rule given by UniSgd, we
consider the $\ell_2$-regularized logistic regression problem
\begin{equation}
    \min_{\bx\in\R^\m}\;
    \frac{1}{N}\sum_{j=1}^N
    \log\bigl(1+\exp(-b_j\langle a_j,\bx\rangle)\bigr)
    +\frac{\lambda}{2}\|\bx\|_2^2,
    \label{eq:logistic_l2_exp}
\end{equation}
on \texttt{news20.binary}\footref{libsvm} from LIBSVM, where $N=12748$,
$\m=62061$, $\lambda=10^{-4}$, and the iteration budget is 1000.
In this case, since the objective function is smooth, the prox term in Algorithm~\ref{adadiff} is not present. Figure~\ref{fig:news20_unisgd_det}
provides an illustration of a setting in which coordinate-wise
adaptation may be advantageous.

\begin{figure*}[t]
    \centering
    \includegraphics[width=0.5\linewidth]{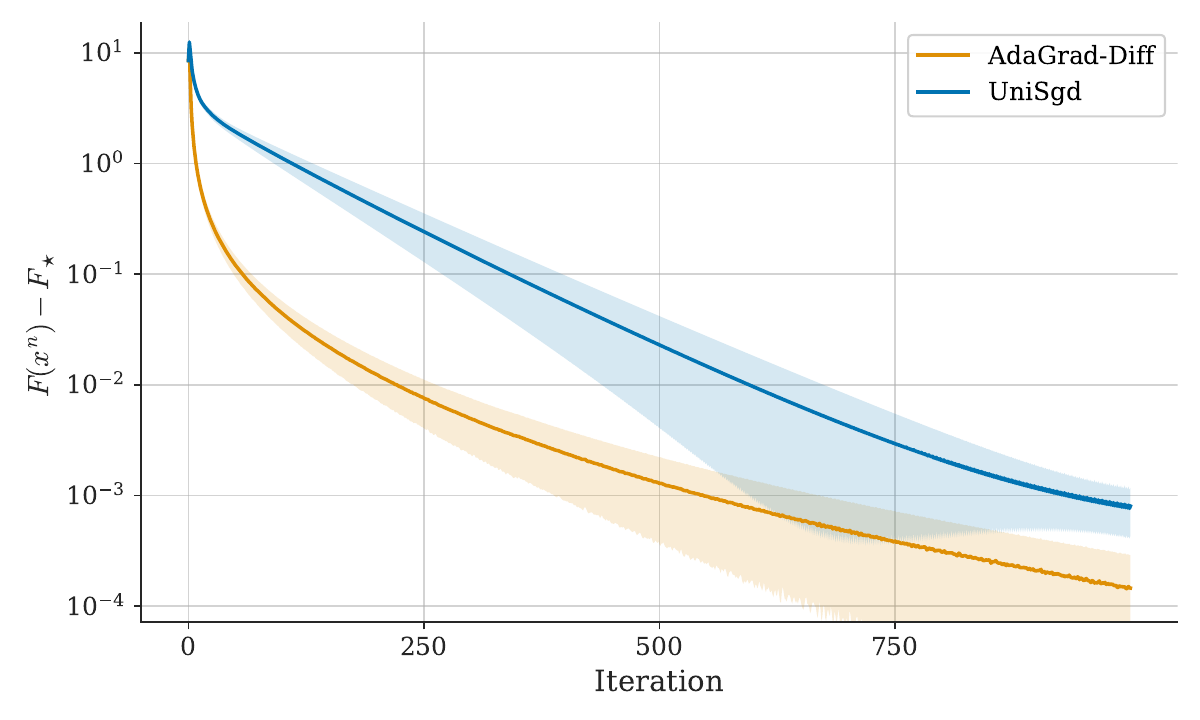}
    \caption{Logistic regression with the \texttt{news20.binary} dataset:
    objective gap of AdaGrad-Diff and UniSgd along the optimization trajectory,
    with optimally tuned parameter $\eta$.}
    \label{fig:news20_unisgd_det}
\end{figure*}

\subsection{Experiments in the stochastic setting}

We revisit \eqref{eq:logistic_l2_exp} on \texttt{news20.binary}, now
comparing mini-batch implementations of AdaGrad and AdaGrad-Diff.
At iteration $n$, a subset $B_n\subseteq[N]$ of size $b$ is sampled
uniformly, independently across iterations, and the gradient estimator is
\begin{equation*}
    \buu^n
    =\frac{1}{b}\sum_{j\in B_n}
    \nabla\log\bigl(1+\exp(-b_j\langle a_j,\bx^n\rangle)\bigr)
    +\lambda\bx^n.
\end{equation*}
In particular, AdaGrad-Diff accumulates successive differences
$\buu^n-\buu^{n-1}$ computed using independently sampled mini-batches,
as in the stochastic oracle model of our analysis.

We use batch sizes $b\in\{512,4096,12748\}$, with $b=12748$
corresponding to full-batch gradients and serving as a deterministic
reference. Each method is run with each batch size under two separate
budgets: $1000$ iterations and $500$ epochs.
The results are averaged over 10 independent initializations:
curves and shaded regions report the average and standard deviation.

\begin{figure*}[t]
    \centering

    \begin{subfigure}[t]{0.98\textwidth}
        \centering
        \begin{minipage}[t]{0.4\linewidth}
            \centering
            \includegraphics[width=\linewidth]{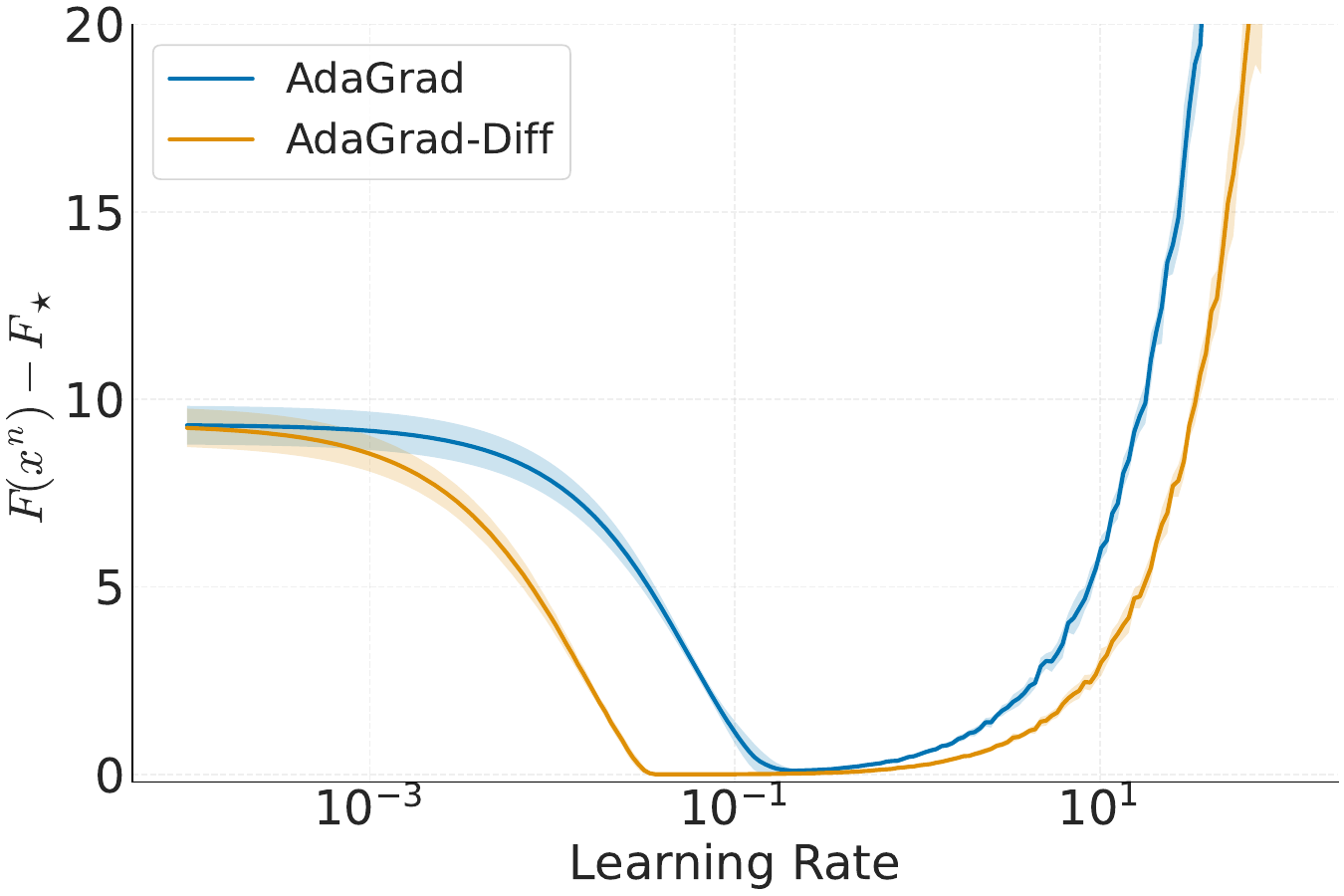}
        \end{minipage}
        \begin{minipage}[t]{0.4\linewidth}
            \centering
            \includegraphics[width=\linewidth]{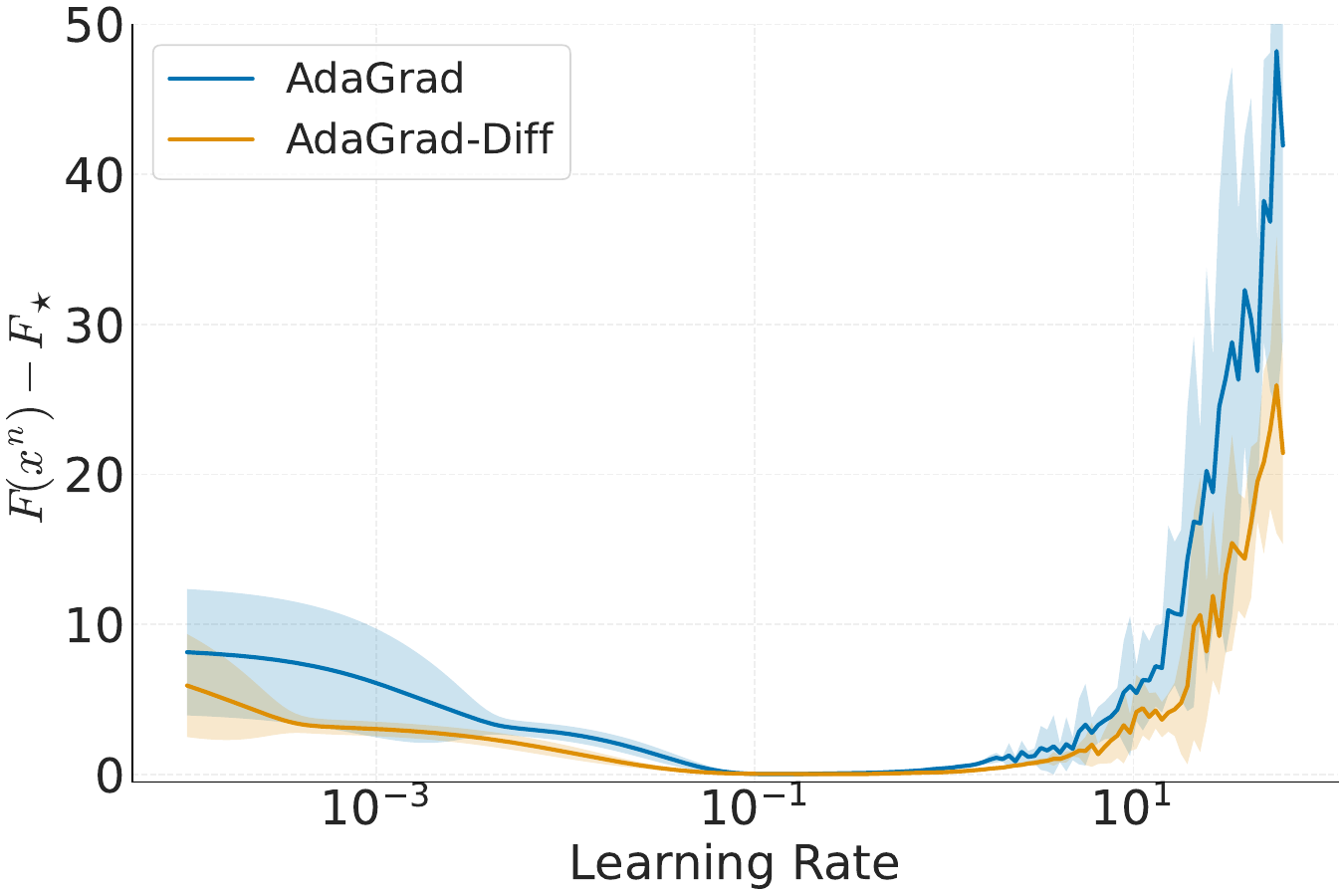}
        \end{minipage}
        \caption{Final objective gaps as a function of $\eta$.}
        \label{fig:robustness}
    \end{subfigure}

    \vspace{0.5em}
    \centering
    \begin{subfigure}[t]{0.98\textwidth}
    \centering
        \begin{minipage}[t]{0.4\linewidth}
            \centering
            \includegraphics[width=\linewidth]{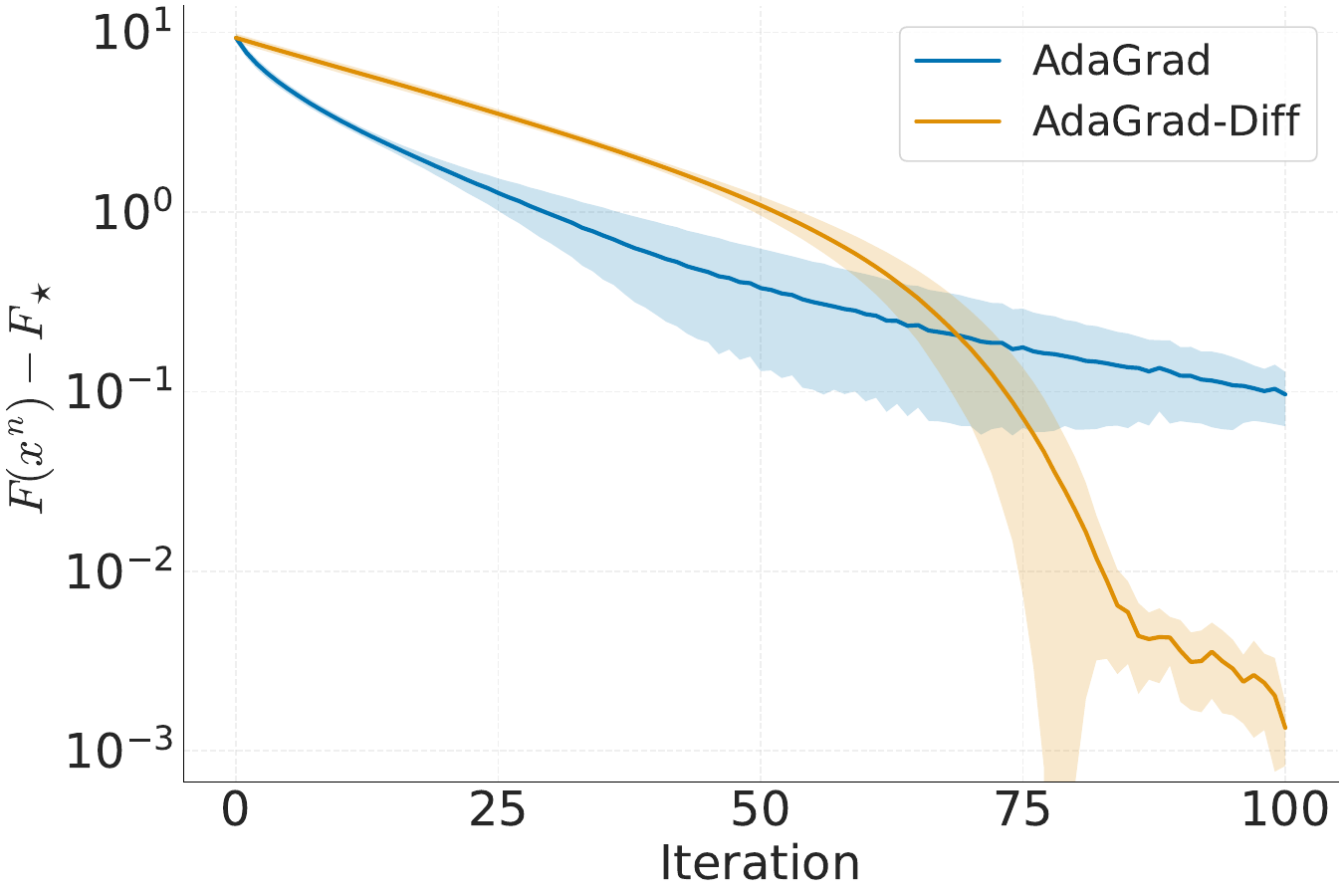}
        \end{minipage}
        \begin{minipage}[t]{0.4\linewidth}
            \centering
            \includegraphics[width=\linewidth]{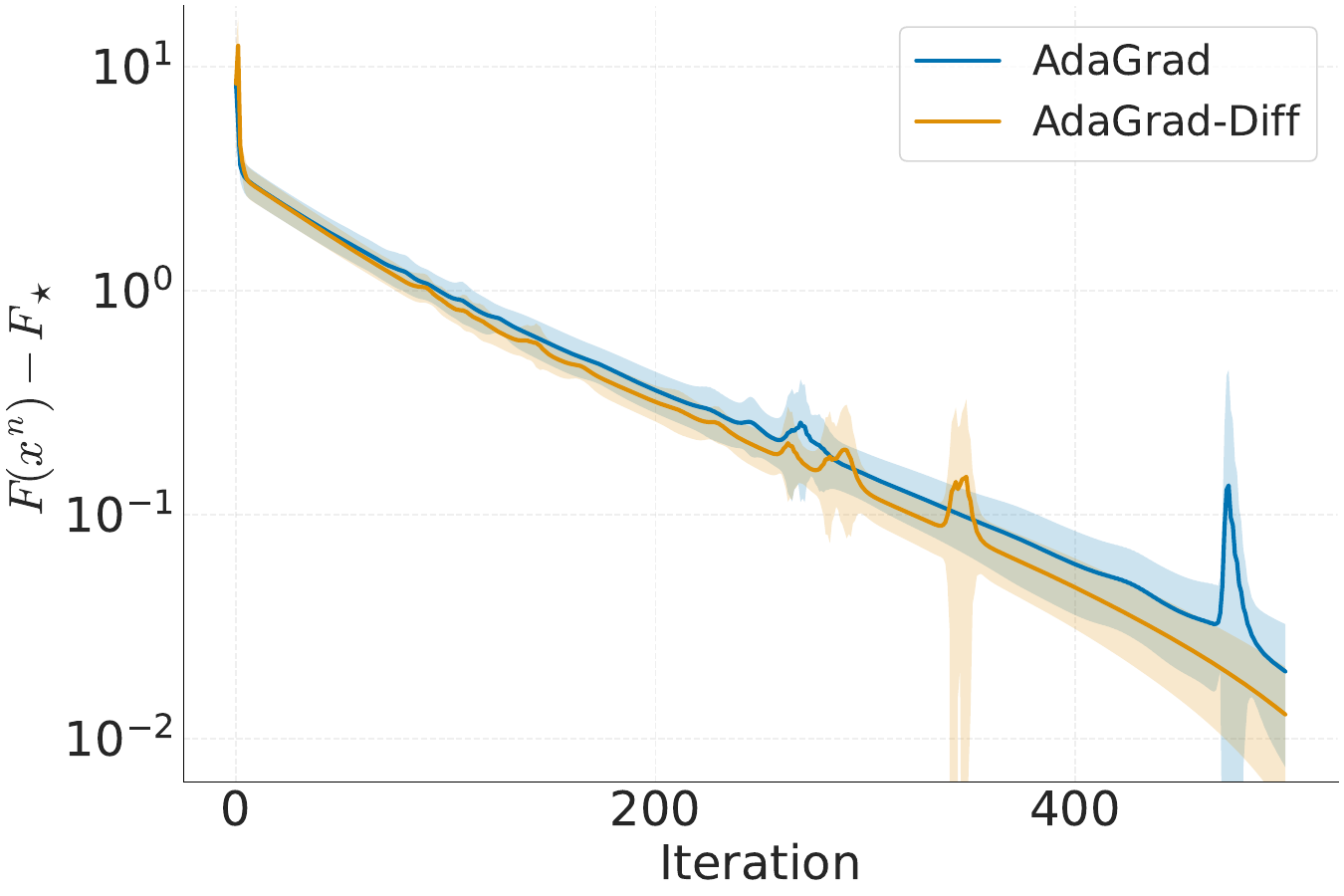}
        \end{minipage}
        \caption{Objective gap along the optimization trajectory with optimally tuned parameter $\eta$.
        }
        \label{fig:comparison}
    \end{subfigure}
    \caption{
        Stepsize sensitivity and optimization performance across representative experiments.
        Columns correspond to LAD regression and logistic regression on \texttt{splice}, respectively.
        The first row shows sensitivity to $\eta$ in terms of final objective gap, while the second row reports the training objective gap along the optimization trajectory using optimally tuned $\eta$.
    }
    \label{fig:stepsize-sensitivity}
    \vspace{-.2truecm}
\end{figure*}

\begin{figure*}
    \centering
    \includegraphics[width=\linewidth]{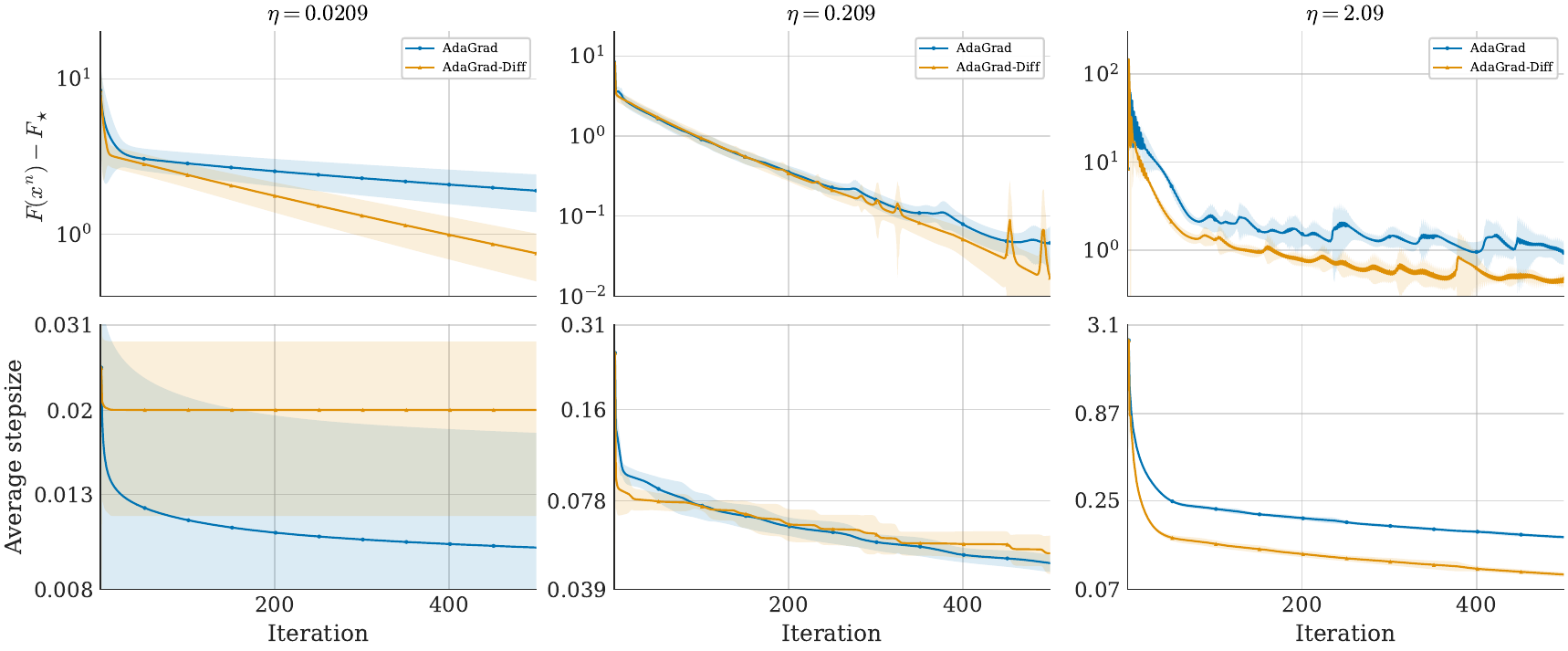}
    \caption{Logistic regression with the \texttt{splice} dataset, for three choices of the stepsize parameter $\eta$.
    The first row reports the objective gap, while the second row shows the corresponding effective stepsize dynamics.}
    \label{fig:splice_opt_step}
\end{figure*}
\begin{figure*}
    \centering

    \begin{subfigure}[t]{0.98\textwidth}
        \centering
        \begin{minipage}[t]{0.4\linewidth}
            \centering
            \includegraphics[width=\linewidth]{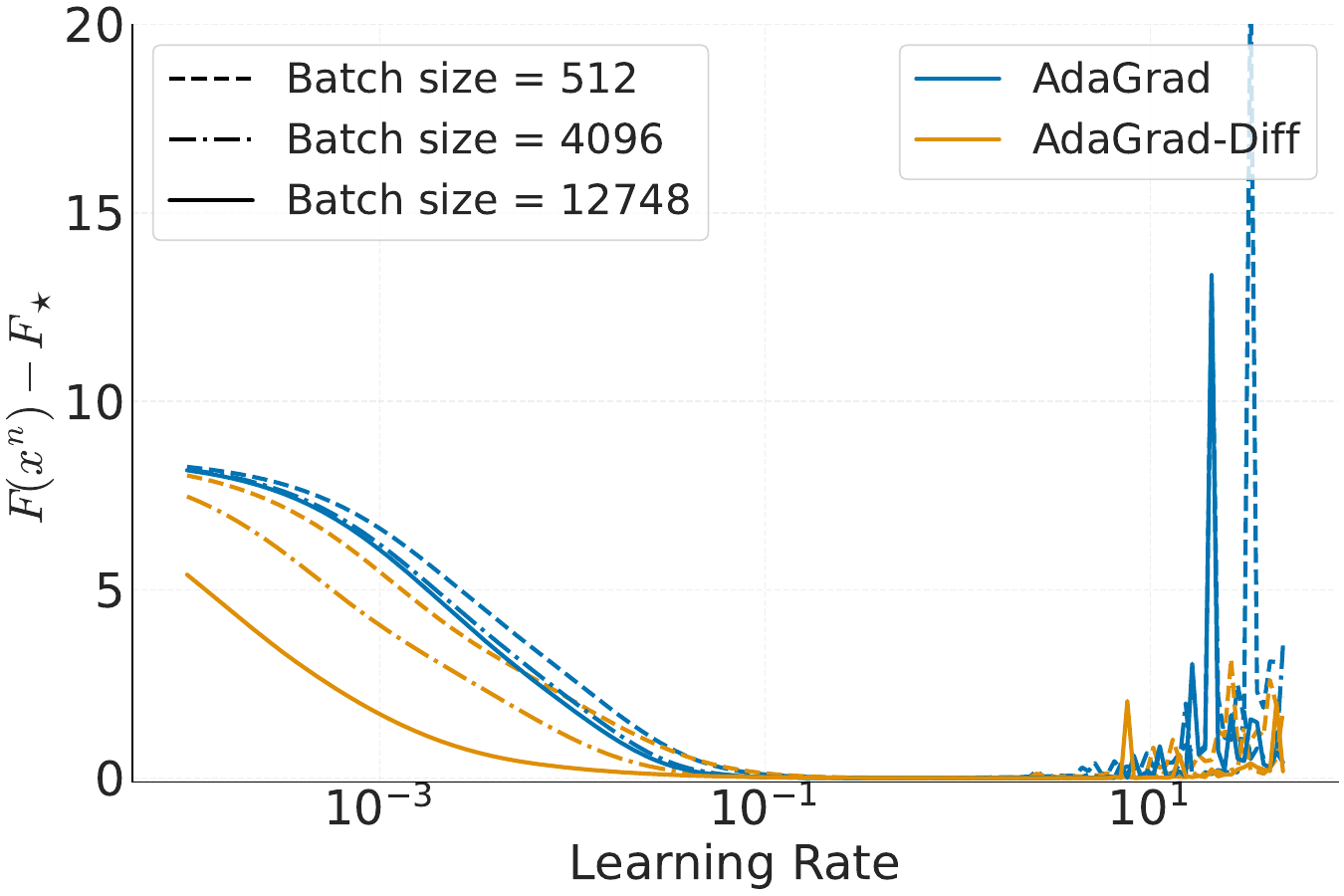}
        \end{minipage}
        \begin{minipage}[t]{0.4\linewidth}
            \centering
            \includegraphics[width=\linewidth]{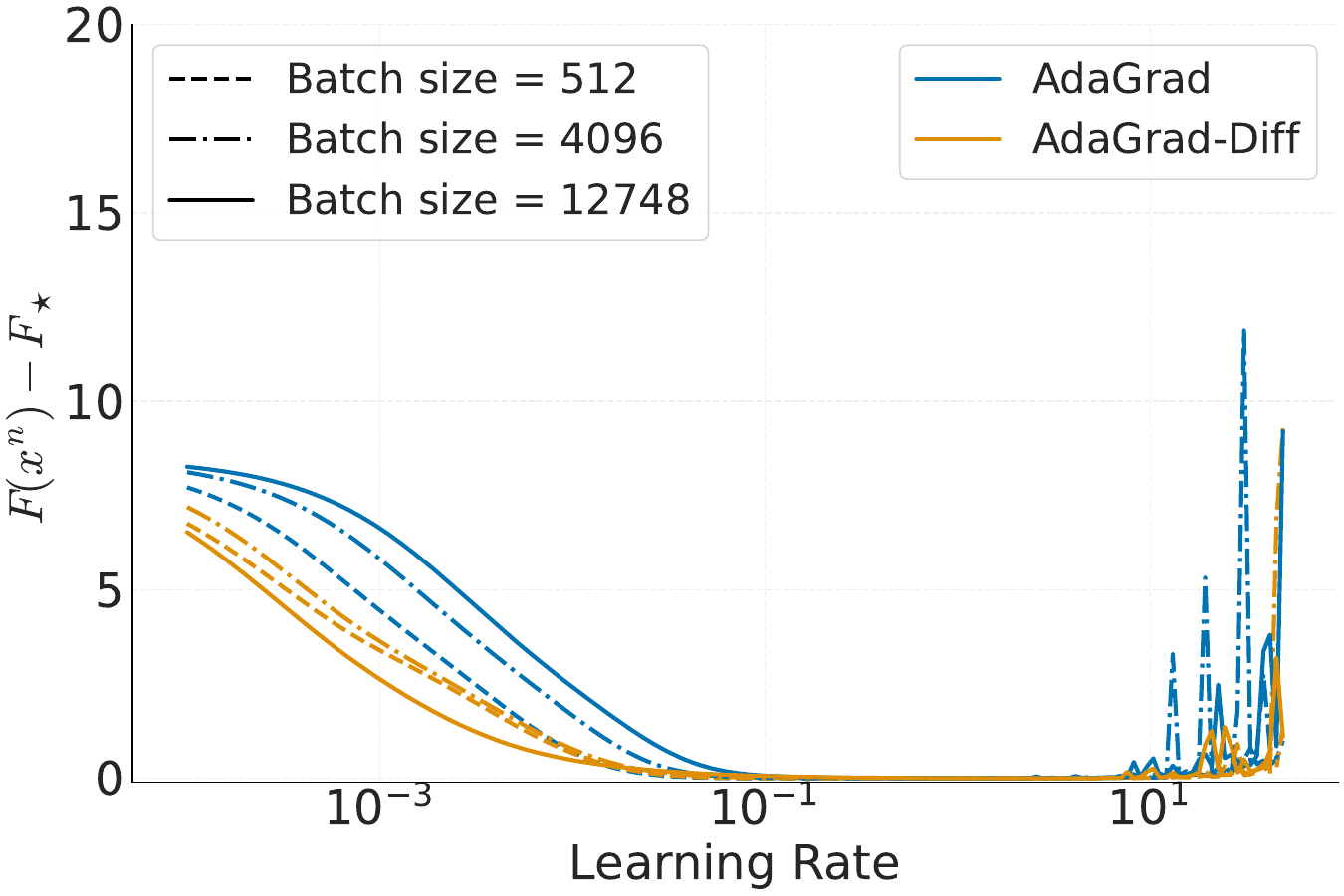}
        \end{minipage}
        \caption{Final objective gaps as a function of $\eta$.}
        \label{fig:news20_stochastic_robust}
    \end{subfigure}

    \vspace{0.5em}

    \begin{subfigure}[t]{0.98\textwidth}
        \centering
        \begin{minipage}[t]{0.4\linewidth}
            \centering
            \includegraphics[width=\linewidth]{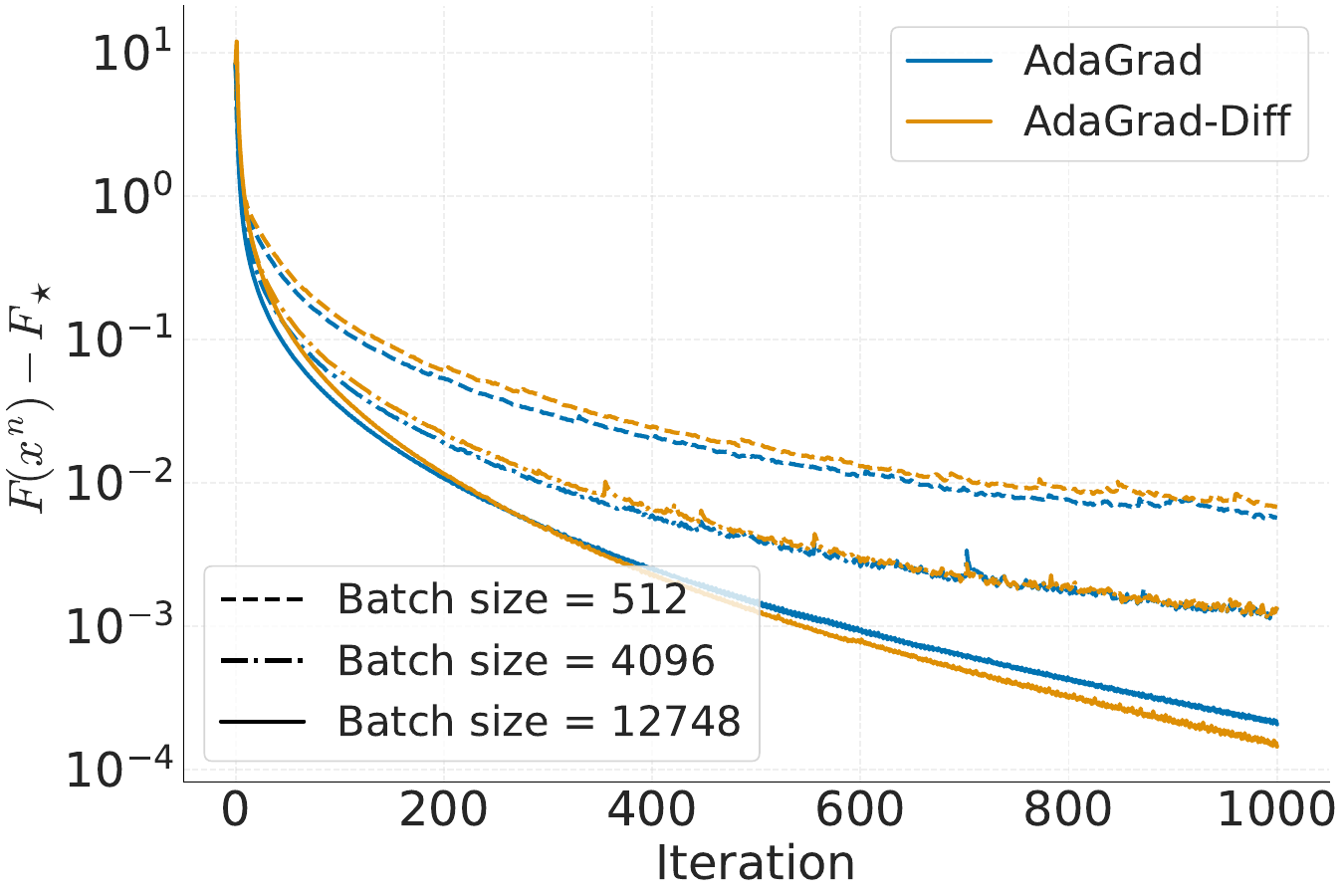}
        \end{minipage}
        \begin{minipage}[t]{0.4\linewidth}
            \centering
            \includegraphics[width=\linewidth]{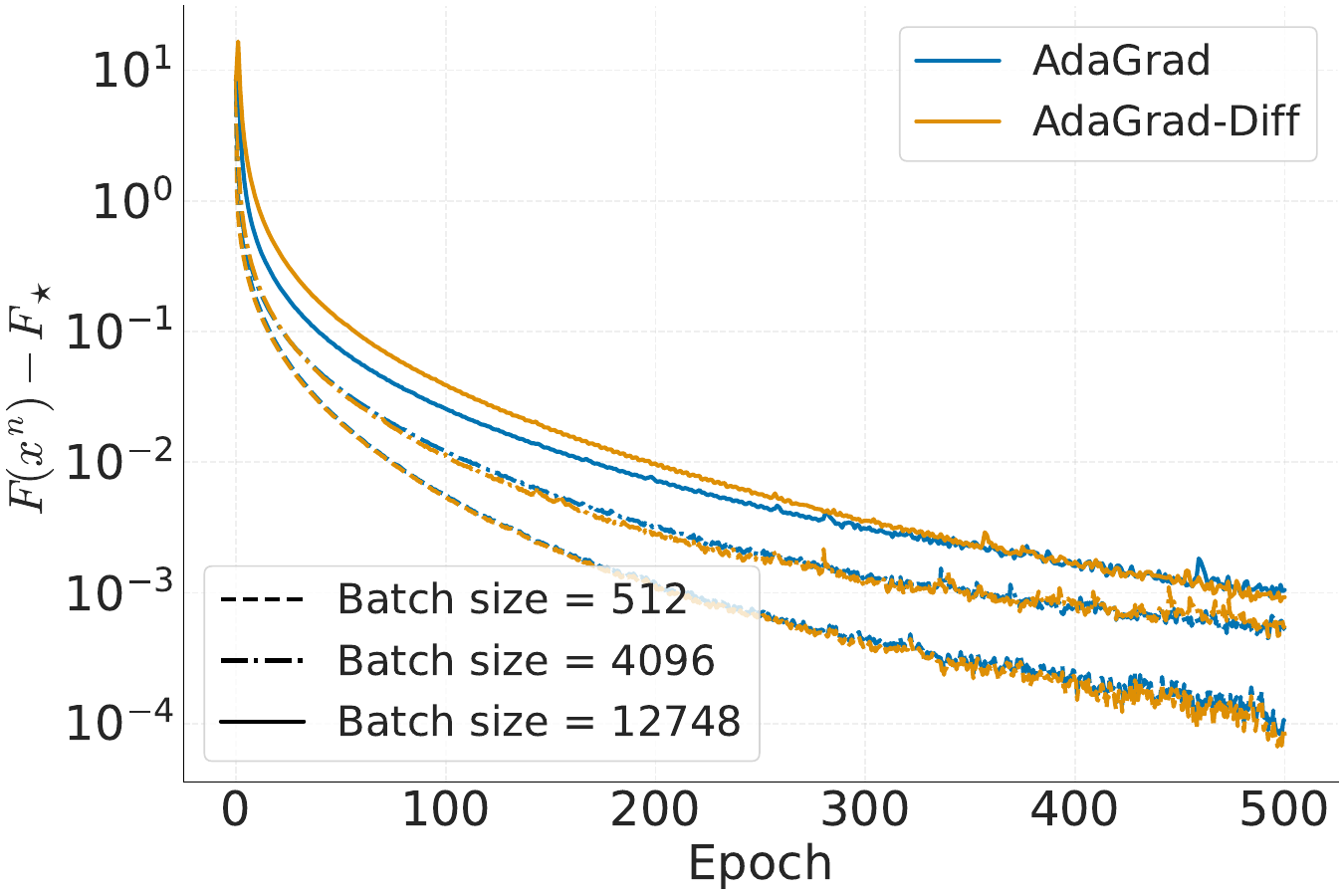}
        \end{minipage}
        \caption{Objective gap along the optimization trajectory with optimally tuned parameter $\eta$.}
        \label{fig:news20_stochastic_opt}
    \end{subfigure}

    \caption{
    Logistic regression on \texttt{news20.binary}.
    AdaGrad and AdaGrad-Diff are compared for batch sizes
    $512$, $4096$, and $12748$.
    The first row shows sensitivity to $\eta$ in terms of final objective gap, while the second row reports the training objective gap along the optimization trajectory using optimally tuned $\eta$.
    The left column considers a fixed number of iterations,
    whereas the right column considers a fixed number of epochs.}
    \label{fig:news20_stochastic}
    \vspace{-.2truecm}
\end{figure*}

\subsection{Numerical results}

\textbf{Minimization performance.}
We evaluate AdaGrad and AdaGrad-Diff using individually tuned values
of $\eta$ for the deterministic and stochastic experiments considered
in this section. As shown in Figures~\ref{fig:comparison}
and~\ref{fig:news20_stochastic_opt}, AdaGrad-Diff exhibits convergence
behavior comparable to, and in several cases better than, that of AdaGrad.

\textbf{Robustness to $\eta$.}
We study the influence of the stepsize parameter $\eta$ on the performance of both methods. The sensitivity plots in Figures~\ref{fig:robustness} and \ref{fig:news20_stochastic_robust} report the final performance obtained after a fixed computational budget for each value of $\eta$. This allows us to assess whether the observed results are tied to a single carefully tuned parameter, or persist over a range of choices. The results show that AdaGrad-Diff converges reliably across a much wider range of choices for $\eta$ compared to standard AdaGrad. In particular, for moderate choices of $\eta$, the method consistently achieves faster convergence. Even in extreme regimes, when $\eta$ is either very small or very large, AdaGrad-Diff maintains meaningful progress along the optimization trajectory, reaching lower objective values than its counterpart. This demonstrates that the accumulation of gradient differences stabilizes the dynamics and enhances robustness to hyperparameter selection, reducing the need for careful tuning.

\textbf{Stepsize dynamics.}
To highlight the mechanism underlying AdaGrad-Diff, we monitor the effective stepsizes produced by AdaGrad and AdaGrad-Diff, for different choices of $\eta$, in the logistic regression experiment with the \texttt{splice} dataset. Figure~\ref{fig:splice_opt_step} shows when the difference-based metric allows larger updates, and when it becomes more conservative in response to rapidly varying gradients.

\textbf{Effect of mini-batch size.}
Figure~\ref{fig:news20_stochastic} compares three batch sizes under fixed
iteration and epoch budgets. Larger batches perform better at a fixed
iteration budget, favoring their use when mini-batch computations can
be perfectly parallelized. At a fixed epoch budget, smaller batches
perform better: the additional updates outweigh the benefit of reduced
stochastic variability.

\section{Conclusions}\label{sec:disc}

We showed that AdaGrad may fail to exploit the H\"older-smoothness of the differentiable term in composite optimization, with convergence rates arbitrarily close to the nonsmooth rate. This occurs because the gradient of the differentiable term need not vanish at a composite solution, so the accumulator may continue to grow and cause the stepsizes to decay excessively. We then studied AdaGrad-Diff, a coordinate-wise extension of the gradient-difference rule introduced within the UniSgd framework. By directly controlling the accumulation, we showed that its stepsizes decay no faster than the classical schedules associated with the regularity of the objective. In the Lipschitz-smooth setting, the same control yields summability of the squared gradient differences and weak convergence of the iterates. The numerical results further show greater robustness than AdaGrad and illustrate a possible benefit of block-wise adaptation over the scalar rule.

\textbf{Future work.}
A natural direction is to extend convergence of the iterates beyond the deterministic Lipschitz-smooth setting. In the deterministic H\"older-smooth case, this likely requires different techniques, since the available estimates do not yield summability of the gradient differences. Under stochastic interpolation, instead, it remains to determine whether such summability can be recovered, potentially enabling a quasi-Fej\'er argument and removing the need to assume boundedness of the iterates a priori. Developing parameter-free variants that remove the need to tune $\eta$ is another important direction. Extensions to nonconvex objectives and the incorporation of exponential moving averages or momentum are also of interest.

\newtheorem{innercustomthm}{{\bf Theorem}}
\newenvironment{customthm}[1]
  {\renewcommand\theinnercustomthm{{\bf #1}}\innercustomthm}
  {\endinnercustomthm}

\newtheorem{innercustomlem}{{\bf Lemma}}
\newenvironment{customlem}[1]
  {\renewcommand\theinnercustomlem{{\bf #1}}\innercustomlem}
  {\endinnercustomlem}
  
  \newtheorem{innercustomcor}{{\bf Corollary}}
\newenvironment{customcor}[1]
  {\renewcommand\theinnercustomcor{{\bf #1}}\innercustomcor}
  {\endinnercustomcor}

\newtheorem{innercustomprop}{{\bf Proposition}}
\newenvironment{customprop}[1]
  {\renewcommand\theinnercustomprop{{\bf #1}}\innercustomprop}
  {\endinnercustomprop}

\newpage
\appendix
\onecolumn
\begin{center}
{\LARGE \bf Appendix} 
\end{center}
Sections~\ref{0} and~\ref{A} contains the missing proofs of the results in Sections~\ref{sec:adagrad-lower-bound} and~\ref{sec:analysis}, respectively. Section~\ref{C} provides a detailed comparison of our convergence rates and variance assumption with those in \citep{rodomanov2024universality}, while
Section~\ref{B} reports additional details about the experimental setup described in Section~\ref{4}.

\section{Missing proofs of Section~\ref{sec:adagrad-lower-bound}}
\label{0}

\begin{customprop}{\ref{prop:prperties_g}}
Let $p>2$ and $c\in \left]0,1\right[$.
The function $g_{p,c}$, defined in \eqref{eq:func_g}, is convex and differentiable with derivative
\begin{equation}{\tag{\ref{derivative}}}
    (\forall\,x\in \R)\quad
    g'_{p,c}(x)
    =
    \begin{cases}
    c\,\operatorname{sign}(x)|x|^{p-1}, & |x|\le 1,\\[0.4em]
    c\,\operatorname{sign}(x), & |x|>1.
    \end{cases}
\end{equation}
 Moreover, 
$g'_{p,c}$ is bounded and Lipschitz-continuous with constant $L=c(p-1)$. Consequently, for every $\nu\in \left]0,1\right]$, 
$g'_{p,c}$ is $\nu$-H\"older continuous with constant $c \max\{2, p-1\}$.
\end{customprop}
\begin{proof}
The fact that $g_{p,c}$ is convex and differentiable with derivative \eqref{derivative} is immediate.
Moreover, the function $g'_{p,c}$ given in \eqref{derivative} is continuous on $\R$ and differentiable on $\R$
with the exception of $\pm 1$ and it holds
\begin{equation*}
    g''_{p,c}(x) =
\begin{cases}
    c(p-1)|x|^{p-2} &\text{if } x\in\left]-1,1\right[\\
    0 &\text{if}\ \R\setminus[-1,1].
\end{cases}
\end{equation*}
This shows that $\abs{g''_{p,c}(x)}\leq c(p-1)$, for every $x\in \R\setminus\{\pm 1\}$, and hence,
by the mean value theorem,  $g'_{p,c}$ is Lipschitz-continuous with constant $c(p-1)$.
It is also clear from \eqref{derivative} that $\abs{g'_{p,c}(x)}\leq c$, that is, $g'_{p,c}$ is bounded.
Then, for every $\nu\in \left]0,1\right]$, $g'_{p,c}$ is  $\nu$-H\"older continuous with constant $\max\{2c, L\}$.
Indeed let $x,y\in \R$. If $\abs{x-y}\leq 1$, we have
\begin{equation*}
\abs{g'_{p,c}(x)-g'_{p,c}(y)}\leq L \abs{x-y} \leq L \abs{x-y}^\nu.
\end{equation*}
On the other hand, if $\abs{x-y}> 1$, then
\begin{equation*}
\abs{g'_{p,c}(x)-g'_{p,c}(y)}\leq 2 c \leq 2c \abs{x-y}^\nu.
\end{equation*}
Hence $g'_{p,c}$ is $\nu$-H\"older continuous with constant $\max\{2c, L\}$. 
\end{proof}

\section{Missing proofs of Section~\ref{sec:analysis}} 
\label{A}

\allowdisplaybreaks

\begin{customlem}{\ref{lemma_recursion2}}
Let \(t,a,b,c,\theta,\rho\ge 0\), with \(a>0\),
\(0\leq\rho<\theta\), and \(\theta-\rho\geq 1\).
If $a t^\theta \le b + c t^\rho$, then
\begin{equation*}
t \leq
\left(\frac{b}{a}\right)^{\frac{1}{\theta}}
+
\left(\frac{c}{a}\right)^{\frac{1}{\theta-\rho}}.
\end{equation*}
\end{customlem}
\begin{proof}
If $t\leq (b/a)^{1/\theta}$, the statement follows. Thus, suppose that
$t>(b/a)^{1/\theta}$. Then, by the inequality
$a t^\theta \le b + c t^\rho$, we derive
\begin{equation*}
t^{\theta-\rho}
\leq
\frac{b}{a}\frac{1}{t^\rho}
+
\frac{c}{a}
\leq
\left(\frac{b}{a}\right)^{1-\frac{\rho}{\theta}}
+
\frac{c}{a}.
\end{equation*}
Hence
\begin{equation*}
t
\leq
\left[
\left(\frac{b}{a}\right)^{1-\frac{\rho}{\theta}}
+
\frac{c}{a}
\right]^{\frac{1}{\theta-\rho}}
\leq
\left(\frac{b}{a}\right)^{\frac{1}{\theta}}
+
\left(\frac{c}{a}\right)^{\frac{1}{\theta-\rho}},
\end{equation*}
where in the last inequality we used the fact that the function
$s\mapsto s^{1/(\theta-\rho)}$ is subadditive on $\R_+$. 
\end{proof}

\begin{customlem}{\ref{lemma:stochastic-gradient-differences}}
% Let Assumption~\ref{ass:unbiased} hold. 
For every $n\in\N$, with $n\geq1$, it holds that
\begin{equation*}
\mathbb E \left[\|\buu^n-\buu^{n-1}\|^2\right]
\leq
\mathbb E \left[\|\buu^n-\nabla f(\bx^n)\|^2\right]
+
2\mathbb E \left[\|\buu^{n-1}-\nabla f(\bx^{n-1})\|^2\right]
+
2\mathbb E \left[\|\nabla f(\bx^n)-\nabla f(\bx^{n-1})\|^2\right].
\end{equation*}
\end{customlem}

\begin{proof}
We start from the decomposition
\begin{equation*}
\buu^n-\buu^{n-1}
=
\buu^n-\nabla f(\bx^n)
+
\nabla f(\bx^n)-\nabla f(\bx^{n-1})
-
\big(\buu^{n-1}-\nabla f(\bx^{n-1})\big).
\end{equation*}
By Assumption~\ref{ass:unbiased}, the first term has zero conditional expectation given $\mathcal F_n$, while the remaining quantity is $\mathcal F_n$-measurable. Hence,
\begin{align*}
\mathbb E \left[\|\buu^n-\buu^{n-1}\|^2\right]
&=
\mathbb E \left[\|\buu^n-\nabla f(\bx^n)\|^2\right]\\
&\quad+
\mathbb E\left\|
\nabla f(\bx^n)-\nabla f(\bx^{n-1})
-\buu^{n-1}+\nabla f(\bx^{n-1})
\right\|^2\\
&\leq
\mathbb E \left[\|\buu^n-\nabla f(\bx^n)\|^2\right]
+
2\mathbb E \left[\|\buu^{n-1}-\nabla f(\bx^{n-1})\|^2\right]\\
&\quad+
2\mathbb E \left[\|\nabla f(\bx^n)-\nabla f(\bx^{n-1})\|^2\right]. 
\end{align*} 
\end{proof}

\section{Further comparison with
\citep{rodomanov2024universality}}\label{C}
As shown in Remark~\ref{rk:comparison}, when $f$ is of the form $f = \mathbb E_\zeta[f_\zeta]$, the refined condition on the variance considered in \citep{rodomanov2024universality} implies
Assumption~\ref{ass:unbiased} with
\begin{equation}\tag{\ref{eq:rodomanov-to-ours}}
    \C=
    2\left(\frac{1+\nu}{\nu}\right)^q
    L^{\frac{2}{1+\nu}},
    \qquad
    \sigma^2=2\sigma_\star^2,
\end{equation}
where $\sigma_\star^2:=\mathbb E \left[\|e(\bx^\star,\zeta)\|^2\right]$. The resulting rates can therefore be compared explicitly. In the setting
above, Corollary~40 of \citep{rodomanov2024universality} gives
\[
\begin{aligned}
    \mathbb E[F(\bar\bx^n)-F_\star]
    \leq{}& \frac{\mathbb E[\max_{1\leq i\leq\m}w_i^0]}{2\eta n}\|\bx^0-\bx^\star\|^2  +
    2\sqrt{30}\sigma_\star Dn^{-\frac12}\\
    & \qquad + \frac{1}{1+\nu}
    \left[
        16^{\frac{1+\nu}{2}}6^{\frac{1-\nu}{2}}
        +
        48^{\frac{1+\nu}{2}}
        \left(\frac{4}{3}\right)^{\frac{1-\nu}{2}}
    \right]
    L D^{1+\nu}n^{-\frac{1+\nu}{2}}.
\end{aligned}
\]
The same problem class satisfies
Assumption~\ref{ass:unbiased} with
$\C$ and $\sigma$ given in~\eqref{eq:rodomanov-to-ours}. To align the
constants with their bounded domain analysis, we repeat our proof using a
diameter bound directly, namely $D$, rather than the radius
$\Delta^\iter_{\max}(\bx^\star)$\footnote{We retain $\Delta^\iter_{\max}(\bx^\star)$ in the
general analysis because it is aligned with the trajectory-dependent scale
used in parameter-free adaptive methods such as in
\citep{ivgi2023dog}.}, and choose $\eta=D$ as in their
analysis. For $\nu\in\left]0,1\right]$, substituting~\eqref{eq:rodomanov-to-ours} gives
\begin{align*}
    \mathbb E[F(\bar\bx^n)-F_\star]
    &\le
    \frac{\mathbb E[\max_{1\leq i\leq\m}w_i^0]}{2\eta n}\|\bx^0-\bx^\star\|^2  +
    2\sqrt{6}\sigma_\star Dn^{-\frac12} + 4\nu^{-\frac{\nu}{1+\nu}}
    L D^{1+\nu}n^{-1}\\
    & \qquad + 2^{1+\nu}
    \left(\frac{1+\nu}{\nu}\right)^\nu
    \left[
        6^{\frac{1+\nu}{1-\nu}}
        +
        2^{\frac{1+\nu}{1-\nu}}
    \right]^{\frac{1-\nu}{2}}
    L D^{1+\nu}n^{-\frac{1+\nu}{2}}.
\end{align*}
Thus, on the problem given by expectation of H\"older-smooth functions considered by
\citep{rodomanov2024universality}, both analyses recover the standard
$\mathcal O(n^{-(1+\nu)/2}+\sigma_\star n^{-1/2})$ rate with explicit,
directly comparable constants. Our result, however, only requires H\"older-smoothness of the population
objective together with
Assumption~\ref{ass:unbiased}. It
requires neither a representation $f=\mathbb E[f_\zeta]$ nor H\"older
regularity of the individual sample functions.
This distinction can affect the guarantee itself. There are simple smooth
problems for which our assumption holds with $\sigma=0$, while
\eqref{ass:rodomanov} fails. Their refined result is then unavailable and
their bounded-variance guarantee retains an $\mathcal O(n^{-1/2})$
stochastic term, whereas ours recovers the deterministic
$\mathcal O(n^{-1})$ rate. 
To see this, let
$\zeta\sim\operatorname{Unif}[1,2]$, and consider
\[
    \widehat f(x,\zeta)
    :=
    \frac{x^2}{2}+(x-\zeta)_+,
    \qquad
    \hat g(x,\zeta)
    :=
     x+\mathbf 1_{\{\zeta<x\}}.
\]
The population objective $f(x):=\mathbb E[\widehat f(x,\zeta)]$ satisfies
\[
    \nabla f(x)
    =
     x+
    \begin{cases}
        0, & x\leq1,\\
        x-1, & 1<x<2,\\
        1, & x\geq2,
    \end{cases}
\]
and is $2$-smooth, with minimizer $x^\star=0$. Moreover, if we define $V(x):=\mathbb E[|\hat g(x,\zeta)-\nabla f(x)\|]^2$, it holds that
\[
    V(x)
    =
    \begin{cases}
        (x-1)(2-x), & 1<x<2,\\
        0, & \text{otherwise}.
    \end{cases}
\]
Since $D_f(x,0)=f(x)\geq x^2/2$, $V(x)\leq\frac{1}{2}D_f(x,0)$,
so Assumption~\ref{ass:unbiased} holds with $q=1,\C=\frac{1}{2},\sigma=0$.
On the other hand, for $x=1+h$ and $y=1$, with $h\in]0,1[$,
\[
    \mathbb E|e(1+h,\zeta)-e(1,\zeta)|^2
    =
    h(1-h),
    \qquad
    D_f(1+h,1)
    =
    h^2.
\]
Hence,
\[
    \frac{
        \mathbb E|e(1+h,\zeta)-e(1,\zeta)|^2
    }{
        D_f(1+h,1)
    }
    =
    \frac{1-h}{h}
    \longrightarrow+\infty
    \qquad\text{as }h\downarrow0,
\]
so their refined two-point condition with $\bar\delta=0$ cannot hold.
At the same time, $\sup_xV(x)=1/4$. Their uniformly bounded variance
analysis is therefore still applicable, and hence a $\mathcal O(n^{-1/2})$ rate is retained.
Our assumption instead holds with $\sigma=0$, so that Theorem~\ref{thm_smooth}
recovers the deterministic $\mathcal O(n^{-1})$ rate.

\section{Details on the experiments}\label{B}

When plotting the objective gap $F(\bx^\iter)-F_\star$, the true optimal value $F_\star$ is generally unknown. We approximate it by keeping track of the smallest objective value observed across all methods, all choices of $\eta$ in the considered grid, all iterations, and all 10 random initializations. In addition, the method exhibiting the fastest decay is run for 10 times the original iteration budget, in order to refine this estimate. All objective gaps are then reported with respect to this empirical value of $F_\star$.
For each problem, the regularization parameter is selected by maximizing the accuracy on a held-out validation set, using an $80/20$ training/validation split.  
The stepsize parameter $\eta$ is selected from a fixed grid, chosen separately for each experiment. When reporting optimization trajectories, each method is run with the value of $\eta$ yielding the smallest final objective value on the training problem. 

\bibliographystyle{plainnat}
\bibliography{refs}

\end{document}